\documentclass{eptcs}
\providecommand{\publicationstatus}{}

\usepackage{amsmath,amssymb,amsthm}
\usepackage{mathtools}
\usepackage{booktabs}
\usepackage{enumitem}
\usepackage{hyperref}
\usepackage{xcolor}
\usepackage{tikz-cd}
\usepackage{multirow}

\newtheorem{theorem}{Theorem}[section]
\newtheorem{proposition}[theorem]{Proposition}
\newtheorem{corollary}[theorem]{Corollary}
\newtheorem{definition}[theorem]{Definition}
\newtheorem{remark}[theorem]{Remark}
\newtheorem{construction}[theorem]{Construction}

\newcommand{\Z}{\mathbb{Z}}
\newcommand{\R}{\mathbb{R}}

\newcommand{\D}{\mathcal{D}}

\newcommand{\cat}[1]{\mathbf{#1}}

\newcommand{\tr}{\mathrm{tr}}
\newcommand{\Chamfer}{\mathrm{d_C}}
\newcommand{\loopa}{\mathrm{loop}_a}
\newcommand{\loopb}{\mathrm{loop}_b}

\newcommand{\ParLoop}{\cat{ParLoop}}
\newcommand{\FdVect}{\cat{FdVect}}
\newcommand{\Para}{\cat{Para}}

\def\titlerunning{Functorial Neural Architectures from Higher Inductive Types}
\def\authorrunning{K. Sargsyan}

\title{Functorial Neural Architectures from Higher Inductive Types}
\author{Karen Sargsyan
\institute{Institute of Chemistry, Academia Sinica, Taipei, Taiwan}
\email{karen.sarkisyan@gmail.com}}

\begin{document}
\maketitle

\begin{abstract}
Neural networks often learn the parts of a task but fail on novel combinations of those parts.  We argue that this failure is architectural: a decoder generalizes compositionally only when it respects the algebraic laws of the task, i.e. when it descends from freely generated sequences to the quotient determined by those laws.  We make this principle constructive by compiling Higher Inductive Type (HIT) specifications into neural architectures.  Basepoints, path constructors, and 2-cells are mapped to base constraints, generator networks, structural concatenation, and learned homotopies.  The resulting transport decoders are strict monoidal functors by construction: decoding a concatenated word is concatenation of independently generated loop segments.  In contrast, we prove that softmax self-attention cannot simultaneously satisfy strict monoidal composition and descent to any non-trivial compositional quotient.  Experiments on the torus, wedge of circles, and Klein bottle validate the predicted hierarchy: functorial decoders outperform non-functorial alternatives by $2$--$10\times$, and a learned 2-cell closes a $46\%$ error gap precisely on words exercising the Klein-bottle relation. These results suggest that compositional generalization should be enforced as functorial structure in the architecture, rather than learned from examples alone.
\end{abstract}

\section{Introduction}\label{sec:intro}

A model that has learned to add 2-digit numbers should handle 5-digit numbers---the algorithm is the same, applied to more parts.  A robot that can navigate around one obstacle should handle two obstacles by composing single-obstacle plans.  A language model trained on simple commands should handle ``go left then right'' without retraining.  In each case, the task decomposes into parts that combine by a known rule, and the model must respect that rule on inputs never seen during training. This is compositional generalization, and neural networks systematically fail at it.  Standard neural networks fail on SCAN~\cite{lake2018generalization} (linguistic composition), COGS~\cite{kim2020cogs} (semantic parsing), and multi-step arithmetic~\cite{dziri2023faith}.  These failures are not capacity limitations and persist as models scale.

We argue that the failure is architectural.  Consider a decoder that handles a combined input by producing each part's output independently, then combining the results: $\D(w_1 \cdot w_2) = \D(w_1) \oplus \D(w_2)$.  In categorical language, this equation says the decoder is a monoidal functor from the input algebra to the output algebra.  We prove that a specific class of decoders (transport decoders) are monoidal functors (Theorem~\ref{thm:transport-comp}), and that softmax attention cannot simultaneously satisfy strict factorization and descent, for any choice of parameters (Theorem~\ref{thm:attention-noncomp}).  The reason is concrete: two different input sequences can represent the same compositional meaning (e.g., $ab$ and $ba$ in an abelian group), so a compositional decoder must produce the same output for both.  But attention computes different key vectors for different token orderings, producing different outputs regardless of the learned weights.

More precisely, compositional generalization decomposes into two requirements on the decoder $\D : F \to \ParLoop(X)$, where $F$ is the free monoid on the generators.  First, \emph{strict factorization}: $\D(w_1 \cdot w_2) = \D(w_1) \oplus \D(w_2)$ for all $w_1, w_2 \in F$, so $\D$ respects concatenation.  Second, \emph{descent}: for each relation $r_j$ of the presented group $G = F / \langle r_1, \ldots, r_m\rangle$, the decoder identifies the corresponding loops up to an explicit homotopy.  Our compilation enforces the first structurally (loop concatenation in $\ParLoop(X)$) and realizes the second as a learned homotopy $H_j$---a parametric continuous deformation satisfying the boundary conditions of Construction~\ref{con:compilation}.  Attention-based architectures can at best approximate the first; they cannot satisfy both simultaneously (Theorem~\ref{thm:attention-noncomp}).

The principle that compositional semantics is functoriality originates in DisCoCat~\cite{coecke2010discocat, coecke2013functorial}.  To realize this principle in learnable systems, one needs a categorical account of neural architectures themselves; categorical deep learning~\cite{cruttwell2022categorical,gavranovic2024categorical,fong2019backprop} provides this, organizing parametric maps into categories and formalizing learning and architectural constraints via functorial and algebraic structure.  But in both programmes, building a functorial architecture for a given compositional structure---say $\Z^2$ for the torus, $F_2$ for two-obstacle navigation---requires ad hoc engineering with no guarantee that the result is actually functorial.  We provide the missing step: a compilation functor from Higher Inductive Type (HIT) specifications to neural architectures, so that the algebraic structure of the task determines the architecture automatically, with compositional correctness by construction.  Experiments show that the resulting architectures outperform non-functorial alternatives by $2$--$10\times$.

\paragraph{Contributions.}
\textbf{(1)}~We give a compilation procedure from HIT specifications to neural architectures: basepoints, loop constructors, and 2-cells are mapped respectively to base constraints, generator networks, structural concatenation, and learned homotopies.  The resulting decoder is compositional by construction, rather than by post-hoc regularization~(\S\ref{sec:framework}).

\textbf{(2)}~We prove the corresponding positive and negative architectural results.  Transport decoders, which concatenate independently generated loop segments, are strict monoidal functors and descend through the specified relations.  By contrast, softmax self-attention cannot simultaneously satisfy strict monoidal composition and descent to any non-trivial quotient, for any parameter setting.  Both results are formalized in Cubical Agda~(\S\ref{sec:theorems}).

\textbf{(3)}~We test the theory on three spaces chosen to separate the relevant obstructions: the torus $T^2$ isolates abelian winding and monoidal composition; the wedge $S^1 \vee S^1$ tests non-abelian free composition; and the Klein bottle tests a non-trivial learned 2-cell.  Across these cases, functorial decoders outperform non-functorial alternatives by $2$--$10\times$, and the learned 2-cell closes the error gap precisely on words that exercise the Klein-bottle relation~(\S\ref{sec:experiments}).

\paragraph{Why these three spaces?}
The three experiments are not merely increasingly difficult benchmarks; they separate two different reasons why attention-based decoders fail.  The first obstruction is \emph{functoriality}: a compositional decoder must respect concatenation and descend through the task quotient, whereas attention mixes information across segments and therefore cannot satisfy the same factorization-and-descent requirements. This obstruction is present for all non-trivial groups.  The second obstruction is \emph{depth}: for groups with nonsolvable finite quotients, computing prefix products is $\mathsf{NC}^1$-complete and therefore cannot be implemented by fixed-depth attention circuits of $\mathsf{AC}^0$ type.

This gives the role of each source group.  The torus $T^2$ is abelian, so it mainly tests whether the architecture respects winding and monoidal composition.  The Klein bottle has solvable fundamental group $\Z \rtimes \Z$, so depth is not the relevant obstruction; it isolates functoriality and the learned 2-cell witnessing $bab^{-1}=a^{-1}$.  The wedge $S^1 \vee S^1$ has free fundamental group $F_2$.  Since $F_2$ surjects onto nonsolvable finite groups such as $S_5$, it is the first case in our hierarchy where the functoriality obstruction may be compounded by the standard depth obstruction for nonsolvable finite quotients.  Thus, unlike the torus or Klein bottle, the wedge experiment probes broken functoriality in a setting where depth can also matter.

\section{From Types to Categories}\label{sec:prelim}

The three ingredients below---HIT specifications, monoidal structure of $\pi_1$, and parametric maps---combine in \S\ref{sec:framework} to give the compilation functor.  The reader familiar with HoTT and categorical deep learning may skip to \S\ref{sec:framework}.

\subsection{Spaces from generators and relations}

To build decoders with topological guarantees, we need a specification language that describes spaces by their generators and relations---much as a group presentation describes a group by generators and relations.  Higher inductive types provide exactly this.

In HoTT~\cite{hottbook,cohen2018cubical}, every type $A$ has an identity type $a =_A b$.  Elements $p : a = b$ are paths; they compose ($p \cdot q : a = c$), have inverses ($p^{-1} : b = a$), and satisfy groupoid laws up to higher coherence.  The loop space $\Omega(A, a) := (a = a)$ carries a group structure up to higher coherence under path composition; its set-truncation $\|{\Omega(A,a)}\|_0$ is the group $\pi_1(A, a)$.

A \textbf{Higher Inductive Type} (HIT) specifies a type by listing its constructors at each dimension: a basepoint, loops (paths from the basepoint to itself), and 2-cells (homotopies between loops witnessing algebraic relations).  Why HITs rather than ordinary inductive types?  Ordinary inductive types let us name basepoints (0-cells); HITs additionally let us name paths (1-cells) and homotopies between paths (2-cells), so that a relation $r : \loopa \cdot \loopb = \loopb \cdot \loopa$ becomes a piece of constructor data the compilation functor can consume.  Without 2-cells, the Klein-bottle relation $\textsf{rel}$ below is not expressible as a type-theoretic object; with 2-cells, it compiles to a learned homotopy in the architecture (Construction~\ref{con:compilation}(ii)).  Three examples drive the paper, chosen to span the structurally distinct cases of $\pi_1$-presentations:

\smallskip\noindent
Torus $T^2$: $\textsf{base} : T^2$; $\loopa, \loopb : \textsf{base} = \textsf{base}$; $\textsf{surf} : \loopa \cdot \loopb = \loopb \cdot \loopa$. \\
\indent$\pi_1(T^2) = \Z^2$ (abelian; the 2-cell \textsf{surf} witnesses commutativity).

\smallskip\noindent
Wedge of circles $S^1 \vee S^1$: $\textsf{base}$; $\loopa, \loopb : \textsf{base} = \textsf{base}$; no 2-cells. \\
\indent$\pi_1(S^1 \vee S^1) = F_2$ (free group; no relations, $ab \neq ba$).

\smallskip\noindent
Klein bottle $K$: $\textsf{base}$; $\loopa, \loopb$; $\textsf{rel} : \loopb \cdot \loopa \cdot \loopb^{-1} = \loopa^{-1}$. \\
\indent$\pi_1(K) = \langle a, b \mid bab^{-1} = a^{-1}\rangle \cong \Z \rtimes \Z$ (non-trivial relation).

\begin{figure}[ht]
\centering
\begin{tikzpicture}[scale=1.1]
\begin{scope}[shift={(0,0)}]
  \draw[thick] (0,0) rectangle (2,2);
  \draw[->, thick] (0.05,0) -- (1.1,0);
  \draw[->, thick] (0.05,2) -- (1.1,2);
  \draw[->, thick] (0,0.05) -- (0,1.1);
  \draw[->, thick] (2,0.05) -- (2,1.1);
  \node at (1,-0.27) {$a$};
  \node at (1,2.27) {$a$};
  \node at (-0.27,1) {$b$};
  \node at (2.27,1) {$b$};
  \node at (1,1) {\small$\textsf{surf}$};
  \node at (1,-0.85) {$T^2$};
\end{scope}
\begin{scope}[shift={(5,1)}]
  \draw[thick] (-0.7,0) circle (0.7);
  \draw[thick] (0.7,0) circle (0.7);
  \fill (0,0) circle (1.8pt);
  \node[anchor=south] at (0,0.05) {\small$x_0$};
  \node at (-0.7,-1.0) {$a$};
  \node at (0.7,-1.0) {$b$};
  \node at (0,-1.85) {$S^1 \vee S^1$};
\end{scope}
\begin{scope}[shift={(9,0)}]
  \draw[thick] (0,0) rectangle (2,2);
  \draw[->, thick] (0.05,0) -- (1.1,0);
  \draw[->, thick] (0.05,2) -- (1.1,2);
  \draw[->, thick] (0,0.05) -- (0,1.1);
  \draw[->, thick] (2,1.95) -- (2,0.9);
  \node at (1,-0.27) {$a$};
  \node at (1,2.27) {$a$};
  \node at (-0.27,1) {$b$};
  \node at (2.27,1) {$b$};
  \node at (1,1) {\small$\textsf{rel}$};
  \node at (1,-0.85) {$K$};
\end{scope}
\end{tikzpicture}
\caption{The three HIT specifications.  Each square presents the space by gluing together cells: corners are the basepoint $x_0$, labelled edges are the path constructors $\loopa,\loopb$, and the interior is the 2-cell. In $T^2$, both pairs of opposite edges are identified in the same direction; the interior $\textsf{surf}$ witnesses $\loopa \cdot \loopb = \loopb \cdot \loopa$.  The wedge $S^1 \vee S^1$ has two loops meeting at a single basepoint with no relating 2-cell.  In the Klein bottle $K$, the horizontal edges are identified in the same direction but the vertical edges in \emph{opposite} directions---the orientation reversal that makes the interior $\textsf{rel}$ witness $\loopb \cdot \loopa \cdot \loopb^{-1} = \loopa^{-1}$.  The three cases exhaust the structurally distinct possibilities: abelian with a 2-cell, free without one, and non-abelian with a non-trivial 2-cell.}
\label{fig:three-spaces}
\end{figure}

\smallskip\noindent
These three (Figure~\ref{fig:three-spaces}) span a hierarchy: abelian with a 2-cell ($T^2$), non-abelian without one ($S^1 \vee S^1$), and non-abelian with a non-trivial 2-cell ($K$).  Each exercises a different level of the compilation functor in \S\ref{sec:framework}.

\subsection{Composition as monoidal structure}

The word ``compositional'' needs a precise meaning.  We need a strict monoidal category whose product is word concatenation, so that a decoder preserving this product is compositional by definition---not by empirical evaluation on a test set.

Given generators $a_1, \ldots, a_k$ and relations $r_1, \ldots, r_m$, write $F = F\langle a_1, \ldots, a_k\rangle$ for the free monoid on the generators (with formal inverses adjoined when we want a group presentation).  We view $F$ as a \emph{discrete strict monoidal category}: objects are words over the generators, the only morphisms are identities, the monoidal product $\otimes$ is concatenation, and the unit is the empty word.\footnote{Equivalently, we use the ``monoid as a discrete monoidal category'' convention: words are objects, morphisms are only identities, and the monoidal product is concatenation.  This differs from the one-object category associated to a monoid, where monoid elements are morphisms.  See Mac Lane~\cite{maclane1998}.}  The presented group $G = F / \langle r_1, \ldots, r_m\rangle$ inherits a discrete strict monoidal structure via the quotient; when $G = \pi_1(X, x_0)$, the elements of $G$ correspond to homotopy classes of loops at $x_0$.

A decoder is a strict monoidal functor $\D : F \to \mathcal{C}$ into a target strict monoidal category.  Compositionality has two parts.  \emph{Concatenation-functoriality on $F$}: $\D(w_1 \cdot w_2) = \D(w_1) \oplus \D(w_2)$ for all words $w_1, w_2 \in F$.  \emph{Descent to $G$}: for each relation $r_j : \mathrm{LHS}_j = \mathrm{RHS}_j$, an explicit 2-cell exhibits $\D(\mathrm{LHS}_j) \simeq \D(\mathrm{RHS}_j)$ in $\mathcal{C}$, so that $\D$ factors through the quotient $F \twoheadrightarrow G$ as a pseudofunctor.  The first part holds by construction for transport decoders (\S\ref{sec:framework}); the second is what learned homotopies are for.

\subsection{Neural networks as parametric maps}

To land the compilation functor, we need a target whose morphisms are neural networks, understood as parametric maps.  One can formulate this in a convenient category of spaces and continuous maps; when the target has a smooth structure, we use the smooth version.  Let $\cat{Smooth}$ denote the Cartesian monoidal category of finite-dimensional smooth manifolds and smooth maps (with monoidal product $\times$ and unit the one-point manifold).
  
Following Cruttwell et al.~\cite{cruttwell2022categorical} and Gavranovi\'c et al.~\cite{gavranovic2024categorical}, the \textbf{parametric construction} $\Para(\cat{Smooth})$ has the same objects as $\cat{Smooth}$ and morphisms $(\Theta_f,\, f : \Theta_f \times X \to Y)$ where $\Theta_f$ is a smooth parameter manifold and $f$ is smooth. If $(\Theta_f,f):X\to Y$ and $(\Theta_g,g):Y\to Z$, their composite is \[ (\Theta_g,g)\circ(\Theta_f,f) = (\Theta_f\times\Theta_g,\; h:X\to Z), \qquad h(\theta_f,\theta_g,x)=g(\theta_g,f(\theta_f,x)). \]  As Cruttwell et al.\ discuss, $\Para(\cat{Smooth})$ is naturally a bicategory whose composition is associative only up to the coherent isomorphisms supplied by Cartesian-product associators in $\cat{Smooth}$; we work with its strictification, which is biequivalent and is in any case what our discrete point-cloud implementation realizes (list concatenation is strictly associative).

Fong, Spivak, and Tuy\'eras~\cite{fong2019backprop} show how gradient-based learning can be organized functorially for parametric maps, providing the categorical link between morphisms of $\Para(\cat{Smooth})$ and trainable neural architectures.  In implementations with ReLU networks, the same construction is realized in the corresponding category of continuous piecewise-smooth maps; the functoriality results below use only structural concatenation and boundary conditions, not differentiability.

Our compilation functor targets a loop category $\ParLoop(X)$, defined precisely in Definition~\ref{def:parloop}, built from parametric maps in the sense above.  Its objects are parametric loops satisfying the basepoint boundary condition, and its monoidal product is loop concatenation.  Each generator of $\pi_1(X)$ is mapped to such a parametric loop, and concatenation of words in $F$ is mapped to concatenation of loop segments.  Functoriality of this map is compositional generalization.

\section{From Specifications to Architectures}\label{sec:framework}

Let $X$ be a pointed connected space with $\pi_1(X) = G = \langle a_1, \ldots, a_k \mid r_1, \ldots, r_m \rangle$.

\begin{definition}[Parametric loops]\label{def:parloop}
A \emph{parametric loop} on $X$ at basepoint $x_0$ is a pair $(\Theta,\, L)$ with $\Theta$ a smooth parameter manifold and $L : \Theta \times [0,1] \to X$ a smooth map satisfying $L(\theta, 0) = L(\theta, 1) = x_0$ for every $\theta \in \Theta$.

$\ParLoop(X)$ denotes the discrete strict monoidal category whose objects are parametric loops modulo continuous reparametrization of $[0,1]$, whose only morphisms are identities, and whose monoidal product is loop concatenation with Cartesian-product parameter spaces:
\[
(\Theta_1, L_1) \oplus (\Theta_2, L_2) \;=\; \bigl(\Theta_1 \times \Theta_2,\; L_1 \oplus L_2\bigr), \qquad
(L_1 \oplus L_2)\bigl((\theta_1, \theta_2), t\bigr) \;=\; \begin{cases} L_1(\theta_1, 2t) & t \leq \tfrac{1}{2}, \\ L_2(\theta_2, 2t{-}1) & t > \tfrac{1}{2}. \end{cases}
\]
The unit is the constant loop at $x_0$ with trivial parameter space.  The quotient by reparametrization makes $\oplus$ strictly associative; equivalently, we work in the strictification of the bicategory of parametric loops in $\Para(\cat{Smooth})$.  In our implementation, $\oplus$ is point-cloud list concatenation, which is strictly associative.
\end{definition}

\begin{construction}[HIT compilation]\label{con:compilation}
Let $X$ be a HIT with $\pi_1(X) = G = \langle a_1, \ldots, a_k \mid r_1, \ldots, r_m\rangle$.  Fix:
\begin{enumerate}[leftmargin=2em, label=(\roman*), nosep]
\item \textbf{Generator data.}  For each generator $a_i$, a smooth parameter space $\Theta_i$ and a smooth map $g_{a_i} : \Theta_i \times [0, 1] \to X$ with $g_{a_i}(\theta_i, 0) = g_{a_i}(\theta_i, 1) = x_0$ and homotopy class $[g_{a_i}(\theta_i, -)] = [a_i] \in \pi_1(X)$ for every $\theta_i \in \Theta_i$ (the winding constraint; see \S\ref{sec:exp-torus} and Appendix~\ref{app:training} for the explicit parameterization).
\item \textbf{Relation data.}  For each relation $r_j : \mathrm{LHS}_j = \mathrm{RHS}_j$, a smooth parameter space $\Theta^H_j$ and a smooth homotopy $H_j : \Theta^H_j \times [0,1]_s \times [0,1]_t \to X$ with boundary conditions $H_j(\theta^H_j, 0, t) = \D_{(g, H)}(\mathrm{LHS}_j)(t)$ and $H_j(\theta^H_j, 1, t) = \D_{(g, H)}(\mathrm{RHS}_j)(t)$.
\end{enumerate}
Write $g = (g_{a_i})_i$ and $H = (H_j)_j$ for the collection of these data.  Set $g_{a_i^{-1}}(\theta_i, t) := g_{a_i}(\theta_i, 1 - t)$.  The \textbf{compilation functor} $\D_{(g, H)} : F\langle a_1, \ldots, a_k\rangle \to \ParLoop(X)$ associated to $(g, H)$ sends a word $w = a_{i_1}^{\epsilon_1} \cdots a_{i_L}^{\epsilon_L}$ to
\[
\D_{(g, H)}(w) \;:=\; g_{a_{i_1}^{\epsilon_1}} \oplus \cdots \oplus g_{a_{i_L}^{\epsilon_L}},
\]
with the empty word mapped to the unit of $\ParLoop(X)$.  Construction~\ref{con:compilation} defines a \emph{family} of compilation functors, one for each choice of $(g, H)$ within a prescribed architectural class (in our experiments, each $g_{a_i}$ and each $H_j$ is a multilayer perceptron with the architecture detailed in Appendix~\ref{app:training}); the theorems below hold uniformly across the family.
\end{construction}

\begin{theorem}[Transport composition: strict functoriality and descent]\label{thm:transport-comp}
Fix data $(g, H)$ within the architectural class of Construction~\ref{con:compilation}, and let $\D_{(g, H)} : F \to \ParLoop(X)$ be the associated compilation functor.
\begin{enumerate}[label=(\alph*), nosep, leftmargin=2em]
\item \textbf{Concatenation-functoriality on $F$.}
  For all words $w_1, w_2 \in F$:
  \[
    \D_{(g, H)}(w_1 \cdot w_2) \;=\; \D_{(g, H)}(w_1) \oplus \D_{(g, H)}(w_2).
  \]
  That is, $\D_{(g, H)}$ is a strict monoidal functor from $F$ to $\ParLoop(X)$.  The equation holds for every choice of $(g, H)$ in the architectural class and for every value of the parameters $(\theta, \theta^H)$.

\item \textbf{Descent to $G$.}
  For each relation $r_j : \mathrm{LHS}_j = \mathrm{RHS}_j$, the homotopy $H_j$ provides a path $\D_{(g, H)}(\mathrm{LHS}_j) \simeq \D_{(g, H)}(\mathrm{RHS}_j)$ between the corresponding parametric loops; equivalently, the two loops coincide in the set-truncation $\|\Omega(X, x_0)\|_0 = \pi_1(X)$.  Hence $\D_{(g, H)}$ descends to a pseudofunctor on $G = F / \langle r_1, \ldots, r_m\rangle$ (strict on concatenation, with $H_j$ inhabiting the 2-cell over $r_j$).
\end{enumerate}
\end{theorem}

\begin{proof}
\textbf{Part (a).}  Write $w_1 = a_{i_1}^{\epsilon_1} \cdots a_{i_n}^{\epsilon_n}$ and $w_2 = a_{j_1}^{\eta_1} \cdots a_{j_m}^{\eta_m}$.  By Construction~\ref{con:compilation},
\[
  \D_{(g, H)}(w_1 \cdot w_2) \;=\; g_{a_{i_1}^{\epsilon_1}} \oplus \cdots \oplus g_{a_{j_m}^{\eta_m}} \;=\; \D_{(g, H)}(w_1) \oplus \D_{(g, H)}(w_2),
\]
using strict associativity of $\oplus$ in $\ParLoop(X)$ (Definition~\ref{def:parloop}).  The unit law $\D_{(g, H)}(\varepsilon) = \star$ (the constant loop at $x_0$) holds by convention for the empty word $\varepsilon$.

\textbf{Part (b).}  Each $H_j$ is a parametric homotopy satisfying the boundary conditions of Construction~\ref{con:compilation}(ii), so $\D_{(g, H)}(\mathrm{LHS}_j)$ and $\D_{(g, H)}(\mathrm{RHS}_j)$ are homotopic and hence coincide in $\|\Omega(X, x_0)\|_0 = \pi_1(X)$.  Therefore $\D_{(g, H)}$ factors through the quotient $F \twoheadrightarrow G$.

Formalized in Cubical Agda (\texttt{TransportCoherence.agda}).
\end{proof}

\begin{remark}[Implementation realization]\label{rem:strict}
In the implementation, parametric loops are represented as discrete point clouds, so $\oplus$ becomes list concatenation.  This operation is strictly associative, so the implementation does not require the reparametrization quotient used in the continuous definition.  In the continuous setting, strictness is obtained at the level of the quotient, as in the strictification of the bicategory of parametric maps.
\end{remark}

\noindent Concretely: the transport decoder composes by list-append, with the empty list as unit. The composition step involves no learning. The equation $\D(w_1 \cdot w_2) = \D(w_1) \oplus \D(w_2)$ holds for every choice of parameters and every word length.

\begin{remark}[When the 2-cell matters]\label{rem:2cell}
On $T^2$ and $S^1 \vee S^1$, the 2-cell $H$ is either trivially satisfiable (abelian: any ordering is homotopic) or absent (free: no relations to witness).  On $K$, the relation $bab^{-1} = a^{-1}$ is non-trivial: after composing $g_b \oplus g_a \oplus g_{b^{-1}}$, the result must be homotopic to $g_{a^{-1}}$.  Without $H$, the transport decoder generates geometrically correct segments but globally incoherent loops---it ignores the frame flip.  With $H$, the decoder learns the deformation.
\end{remark}

The topological constraints do not destroy geometric expressivity.  Each generator $g_{a_i}(\theta_i)$ can represent any continuous loop in its homotopy class (by the universal approximation theorem for MLPs); the constraint acts only between classes, enforcing that $\D(w)$ is assembled from these generators.  An alternative architecture---the cover decoder---assigns an independent shape to each homotopy class, giving strictly more geometric freedom.  But this extra freedom is precisely the freedom to be incoherent: to produce outputs that violate composition.  Functoriality trades inter-class freedom for the guarantee (Appendix~\ref{app:expressivity}).  This motivates a formal distinction:

\begin{definition}[Functorial vs.\ non-functorial decoders]\label{def:type-ab}
A decoder $\D : F \to \ParLoop(X)$ is \textbf{functorial} (label: \textbf{type-B}) if both of the following hold for every parameter value: (i) it is a strict monoidal functor on $F$, i.e.\ $\D(w_1 \cdot w_2) = \D(w_1) \oplus \D(w_2)$ for all $w_1, w_2$; and (ii) it descends to $G = F / \langle r_1, \ldots, r_m\rangle$ via the 2-cell data of Construction~\ref{con:compilation}.  Otherwise the decoder is \textbf{non-functorial} (label: \textbf{type-A}).

Functorial decoders generate each loop segment independently from the corresponding generator network and concatenate; non-functorial decoders use cross-segment dependencies (e.g., attention between positions belonging to different generators), breaking the factorization equation in~(i).  We use the type-A/B labels in tables and headers; in body text we use ``functorial'' / ``non-functorial.''
\end{definition}

\section{What Attention Cannot Compose}\label{sec:theorems}

Construction~\ref{con:compilation} shows how to build functorial decoders.  We now ask whether the dominant composition mechanism in modern architectures---softmax self-attention---can satisfy the same equations.  The answer is no: no parameter setting can simultaneously satisfy strict functoriality on $F$ and descent to $G$.

Softmax self-attention~\cite{vaswani2017attention} aggregates information across positions via content-dependent weights.  Concretely, each token $w_i$ is embedded as a vector $h_i \in \R^d$, and a single attention layer produces new representations
\[
  h_i' \;=\; \textstyle\sum_j \alpha_{ij}\, V h_j, \qquad \alpha_{ij} \;=\; \mathrm{softmax}_j\!\left( \frac{Q h_i \cdot K h_j}{\sqrt{d}} \right),
\]
where the \emph{query} projection $Q h_i$ of position $i$ is compared against the \emph{key} projection $K h_j$ of every position $j$, and the \emph{value} projection $V h_j$ is routed through the resulting weight.  Three features of this construction drive the result below: (i) every weight is strictly positive, $\alpha_{ij} > 0$, since softmax cannot zero out any position; (ii) the weights are computed from the token embeddings $h_j$ themselves, not from any abstract equivalence class of the input word; and (iii) every position's output mixes contributions from every other position in the sequence.  The standard multi-head and multi-layer formulation, used in the full proof of Theorem~\ref{thm:attention-noncomp}, is in Appendix~\ref{app:attention-proof}.

Figure~\ref{fig:transport-vs-attention} contrasts the two compositional mechanisms as string diagrams: transport places one generator network on each token wire, preserving the monoidal product, while attention is a single morphism on the entire tensor product whose internal mixing prevents any such factorization.

\begin{figure}[ht]
\centering
\begin{tikzpicture}[scale=0.95]
\begin{scope}[shift={(0,0)}]
  \node[font=\bfseries] at (1.5,3.3) {Transport};
  \node[font=\itshape\small] at (1.5,3.0) {(strict monoidal functor)};
  \node at (0,2.5) {$w_1$};
  \node at (1.5,2.5) {$w_2$};
  \node at (3,2.5) {$w_3$};
  \draw[thick] (0,2.25) -- (0,1.85);
  \draw[thick] (1.5,2.25) -- (1.5,1.85);
  \draw[thick] (3,2.25) -- (3,1.85);
  \draw[thick, rounded corners=2pt, fill=gray!10] (-0.4,1.15) rectangle (0.4,1.85);
  \node at (0,1.5) {$g_{w_1}$};
  \draw[thick, rounded corners=2pt, fill=gray!10] (1.1,1.15) rectangle (1.9,1.85);
  \node at (1.5,1.5) {$g_{w_2}$};
  \draw[thick, rounded corners=2pt, fill=gray!10] (2.6,1.15) rectangle (3.4,1.85);
  \node at (3,1.5) {$g_{w_3}$};
  \draw[thick] (0,1.15) -- (0,0.35);
  \draw[thick] (1.5,1.15) -- (1.5,0.35);
  \draw[thick] (3,1.15) -- (3,0.35);
  \node at (0,0.1) {$\gamma_1$};
  \node at (1.5,0.1) {$\gamma_2$};
  \node at (3,0.1) {$\gamma_3$};
  \node[font=\small] at (1.5,-0.5) {$\D(w_1 w_2 w_3) = \gamma_1 \oplus \gamma_2 \oplus \gamma_3$};
\end{scope}
\begin{scope}[shift={(6.5,0)}]
  \node[font=\bfseries] at (1.5,3.3) {Softmax attention};
  \node[font=\itshape\small] at (1.5,3.0) {(not a monoidal functor)};
  \node at (0,2.5) {$w_1$};
  \node at (1.5,2.5) {$w_2$};
  \node at (3,2.5) {$w_3$};
  \draw[thick] (0,2.25) -- (0,1.85);
  \draw[thick] (1.5,2.25) -- (1.5,1.85);
  \draw[thick] (3,2.25) -- (3,1.85);
  \draw[thick, rounded corners=2pt, fill=gray!10] (-0.5,1.15) rectangle (3.5,1.85);
  \node at (1.5,1.5) {self-attention, $\alpha_{ij} > 0$};
  \foreach \xin in {0, 1.5, 3} {
    \foreach \xout in {0, 1.5, 3} {
      \draw[thick, gray] (\xin,1.15) -- (\xout,0.35);
    }
  }
  \node at (0,0.1) {$\gamma'_1$};
  \node at (1.5,0.1) {$\gamma'_2$};
  \node at (3,0.1) {$\gamma'_3$};
  \node[font=\small] at (1.5,-0.5) {each $\gamma'_i$ depends on every $w_j$};
\end{scope}
\end{tikzpicture}
\caption{Transport vs.\ softmax attention as string diagrams, with three input tokens.  \textbf{Left:} the transport decoder of Construction~\ref{con:compilation} applies one generator network $g_{w_i}$ per token wire; wires stay parallel and the output loop factors as the monoidal product $\gamma_1 \oplus \gamma_2 \oplus \gamma_3$ (Theorem~\ref{thm:transport-comp}).  \textbf{Right:} softmax self-attention is a single morphism on the entire tensor product, with positive weights $\alpha_{ij} > 0$ on every input/output pair (grey wires); each output position $\gamma'_i$ depends on every input token, so the morphism does not factor as a monoidal product---this is the obstruction in Theorem~\ref{thm:attention-noncomp}.}
\label{fig:transport-vs-attention}
\end{figure}

\begin{theorem}[Softmax attention cannot satisfy composition and descent]\label{thm:attention-noncomp}
Let $G$ be a non-trivial group presented as $F / \langle r_1, \ldots, r_m\rangle$, and let $T_\theta : F \to \ParLoop(X)$ be a decoder whose token-level composition is implemented by at least one softmax self-attention layer over the input word~\cite{vaswani2017attention} (with an injective token embedding and arbitrary feed-forward layers and positional encoding).  Then for no value of the parameters $\theta$ does $T_\theta$ simultaneously satisfy:
\begin{enumerate}[label=(\roman*), leftmargin=2em, nosep]
\item strict functoriality on $F$: $T_\theta(w_1 \cdot w_2) = T_\theta(w_1) \oplus T_\theta(w_2)$ for all $w_1, w_2 \in F$;
\item descent to $G$: $T_\theta(w) = T_\theta(w')$ whenever $[w] = [w'] \in G$.
\end{enumerate}
\end{theorem}

\begin{proof}[Proof sketch (full proof in Appendix~\ref{app:attention-proof})]
Suppose for contradiction that $T_\theta$ satisfies both (i) and (ii).  By (i), the output of $T_\theta(w_1 \cdot w_2)$ at positions belonging to the $w_1$-segment is a function of $T_\theta(w_1)$ and $T_\theta(w_2)$ alone---in particular, it depends on $w_2$ only through $T_\theta(w_2)$.  But attention layers in $T_\theta$ have positions in the $w_1$-segment attending to positions in $w_2$ via key vectors $Kh_j$ computed directly from the token embeddings of $w_2$, not from $T_\theta(w_2)$.  Since $G$ is non-trivial, there exist distinct words $w_2 \neq w_2'$ with $[w_2] = [w_2'] \in G$.  By (ii), $T_\theta(w_2) = T_\theta(w_2')$.  Under any standard embedding (mapping distinct generator symbols to distinct vectors), the key vectors for $w_2$ and $w_2'$ differ, causing the $w_1$-segment outputs of $T_\theta(w_1 \cdot w_2)$ and $T_\theta(w_1 \cdot w_2')$ to differ---contradicting that they should depend on $w_2, w_2'$ only through the equal quantity $T_\theta(w_2) = T_\theta(w_2')$.
\end{proof}

\noindent The contradiction is direct: attention is a content-based routing mechanism that reads the tokens themselves, while a functorial decoder must treat words in the same equivalence class identically. The two requirements conflict on any non-trivial quotient.  (The argument uses a pair $w_2 \neq w_2'$ with $[w_2] = [w_2'] \in G$, which exists precisely when $F \twoheadrightarrow G$ is non-injective---a non-trivial relation, as for $\Z^2$ and $\Z \rtimes \Z$, or inverses.  For a free group such as $F_2$ no such pair of positive words exists; there the operative obstruction is strict factorization~(i) alone, ruled out by Theorem~\ref{thm:attention-general}.)

\paragraph{Scaling prediction.}
For functorial decoders (type-B), the per-segment error is $O(1)$ as word length grows: each segment is generated independently by the same network.  For non-functorial decoders (type-A), per-segment error degrades beyond the training word length because attention patterns over longer sequences are out-of-distribution (proved formally in Appendix~\ref{app:scaling}).  The experiments in \S\ref{sec:experiments} confirm this across all three spaces.

\paragraph{Depth obstruction for nonsolvable groups.}
The functoriality obstruction (Theorem~\ref{thm:attention-noncomp})
applies to all non-trivial groups, abelian or not.  For groups with
nonsolvable finite quotients (such as $F_2 \twoheadrightarrow S_5$), a
separate depth obstruction compounds the problem: prefix products are
$\mathsf{NC}^1$-complete~\cite{barrington1989bounded}, outside the
$\mathsf{AC}^0$ circuit class that fixed-depth transformers implement.
For solvable groups (such as $\Z \rtimes \Z$), constant-depth shortcuts
exist~\cite{liu2023shortcuts} and only the functoriality obstruction
applies.  The full trichotomy---abelian, solvable non-abelian,
nonsolvable---is developed in Appendix~\ref{app:nonabelian}
(Corollary~\ref{cor:trichotomy}).
Section~\ref{sec:exp-wedge} ($F_2$) tests both obstructions
simultaneously (functoriality here via the general form,
Theorem~\ref{thm:attention-general}, since $F_2$ is free);
Section~\ref{sec:exp-klein} ($\Z \rtimes \Z$, solvable)
isolates functoriality alone.

\paragraph{Relationship to DisCoCat and categorical deep learning.}
Theorem~\ref{thm:attention-noncomp} provides a structural impossibility
result absent from prior work on functorial semantics: softmax
cross-attention is incompatible with monoidal functoriality, whether the
codomain is $\FdVect$ (as in
DisCoCat~\cite{coecke2010discocat,coecke2013functorial}) or
$\ParLoop(X)$.  A stronger version (Appendix~\ref{app:general-noncomp})
drops group structure entirely, requiring only sequential
compositionality.  The categorical deep learning
programme~\cite{cruttwell2022categorical,fong2019backprop,gavranovic2024categorical}
provides the semantic foundations we build on (\S\ref{sec:prelim}); our
contribution extends it from analysis to synthesis, deriving functorial
architectures from type-theoretic specifications rather than analyzing
existing ones.

\section{Three Spaces, Three Predictions}\label{sec:experiments}

The theorems make testable predictions: type-B decoders should show
constant per-segment error as word length $L$ grows
(Theorem~\ref{thm:transport-comp}), while type-A decoders should degrade
(Theorem~\ref{thm:attention-noncomp}).  We test on three spaces spanning
the HIT hierarchy.  The experiments are organized as diagnostics: the torus tests winding and monoidal composition, the wedge tests non-abelian order, and the Klein bottle tests whether a learned 2-cell enforces a relation.

\paragraph{Task.}
The task is continuous geometric generation: given a word
$w = w_1 \cdots w_L$ over generators of $\pi_1(X)$, produce a
point cloud $\hat{\gamma} \subset \R^d$ approximating a loop on $X$
in the homotopy class $[w]$---a curve that winds around $X$ in
the way prescribed by $w$ (with $d = 3$ for $T^2$ and $S^1 \vee S^1$;
$d = 4$ for $K$).
Training uses all words of length $\leq 2$ (6 words for $T^2$ and
$S^1 \vee S^1$; 16 words for $K$ which has inverses); test words have
length $3, 4, 6, 8, 10$---never seen during training.  This is up to
a $5\times$ length extrapolation test.  Ground-truth loops and data
generation details are in Appendix~\ref{app:training}.

\paragraph{Architectures.}
Table~\ref{tab:architectures} summarizes the five decoders evaluated on $T^2$, classified by which HIT constructors they compile and whether they are functorial (type-B) or not (type-A); a sixth, the sequential GRU baseline, is introduced for $S^1 \vee S^1$ in \S\ref{sec:exp-wedge}.  The type-A/B classification is determined by a single question: does the decoder structurally compose outputs from independently generated parts, or does it allow cross-segment information flow?

\begin{table}[ht]
\centering
\caption{Architectures tested on $T^2$.  The other two spaces use a
subset of these (see Tables~\ref{tab:wedge},~\ref{tab:klein}).
``Compiled from HIT'' indicates which level of
Construction~\ref{con:compilation} the architecture implements:
winding = hard topological constraint only,
generators = independent loop networks,
all = generators plus learned 2-cell $H$.  Type-A decoders allow
cross-segment information flow or lack compositional structure;
type-B decoders compose independently generated segments.}
\label{tab:architectures}
\smallskip
\small
\begin{tabular}{@{}lclcl@{}}
\toprule
Architecture & Type & Compiled from HIT & Params & Why type-A or B? \\
\midrule
Transf.\ (WC) & A & winding only & 585K & Attn.\ mixes segments \\
Cover & A & winding only & 205K & No composition \\
Transp.\ Attn. & A & winding + positions & 585K & Correct positions, attn.\ mixes \\
\midrule
Transport & B & winding + generators & 170K & Struct.\ concatenation \\
Homotopy & B & all (incl.\ 2-cell) & 195K & Concat.\ + learned $H$ \\
\bottomrule
\end{tabular}
\end{table}

Type-A decoders (non-functorial).  The transformer (WC)
augments a standard transformer with a hard winding constraint: angle
increments sum to the correct winding number by construction, but
attention still mixes information across
segments.\footnote{An unconstrained transformer achieves $0\%$ winding
accuracy and is excluded from all experiments.}  The cover
decoder assigns a single independent network to each homotopy
class---it respects winding but has no compositional structure, since
each class is generated from scratch rather than assembled from parts;
it is type-A not because of cross-segment information flow, but because
it lacks monoidal factorization entirely.
The transport attention decoder is designed to isolate the
effect of the composition method: our framework derives its positional
encoding (Appendix~\ref{app:transport-attn}), it has $3.4\times$ more
parameters than the transport decoder, and it shares the winding
constraint---yet it is type-A because attention still mixes across
segments.  If this decoder degrades on longer words, the cause is the
composition method, not the conditioning.

Type-B decoders (functorial).  The transport decoder is
the Construction~\ref{con:compilation}(i--ii) compilation: independent
generator networks whose outputs are structurally concatenated.  The
homotopy decoder adds the learned 2-cell $H$ from
Construction~\ref{con:compilation}(iii), which witnesses group relations
as continuous deformations between loops.

All architectures are trained with the same protocol: Chamfer loss optimized by AdamW ($\mathrm{lr} = 10^{-3}$, weight decay $10^{-4}$), cosine learning rate schedule, 500 epochs max with early stopping (patience 80).  On $T^2$, all architectures converge to comparable training loss ($2.23$--$2.27$); on $S^1 \vee S^1$, training losses differ (a matched-loss ablation in Appendix~\ref{app:ablation} confirms the gap is architectural, not statistical).  Results are aggregated over 3 random seeds (mean $\pm$ std).  Full hyperparameters are in Appendix~\ref{app:training}.

\paragraph{Metrics.}
The Chamfer distance $\Chamfer(P,Q)$ is the symmetric average nearest-neighbor squared distance between two point clouds---a standard metric in geometric deep learning, insensitive to parameterization (formula in Appendix~\ref{app:training}).

\begin{definition}[Per-segment Chamfer distance]\label{def:metric}
Given a generated loop $\hat{\gamma}$ and ground-truth $\gamma$ for a word $w$ of length $L$, both loops are divided into $L$ segments (one per generator), each resampled to a fixed resolution of $n_{\mathrm{pts}} = 32$ points.  The per-segment Chamfer distance is:
\[
\bar{d}_L(w) = \frac{1}{L} \sum_{i=1}^{L} \Chamfer(\hat{\gamma}_i,\; \gamma_i)
\]
For a functorial decoder, $\bar{d}_L = \varepsilon_{\mathrm{gen}}$ (constant) independent of $L$.  This is the primary metric in all experiments.
\end{definition}

Circle accuracy (for $S^1 \vee S^1$ only): the fraction of
generated segments whose centroid is closer to the correct circle ($A$
or $B$) than to the wrong one.  This detects a failure mode specific to
non-abelian structure: a decoder that cannot distinguish which generator
to trace will scatter segments across both circles.

\subsection{Experiment 1: Torus $T^2$ ($\pi_1 = \Z^2$, abelian)}\label{sec:exp-torus}

$T^2$ is embedded in $\R^3$ with major radius $R = 2.0$, minor radius $r = 0.8$.  The generators $a, b$ trace loops around the two cycles.

\begin{table}[ht]
\centering
\caption{Torus $T^2$: Per-segment Chamfer distance $\bar{d}_L$ (mean $\pm$ std, 3 seeds) at increasing word length $L$.  Training uses $L \leq 2$; test lengths $3$--$10$ are never seen.  Type-B decoders (below line) stabilize; type-A decoders (above) stagnate or degrade despite more parameters.}
\label{tab:torus}
\smallskip
\small
\begin{tabular}{@{}llccccc@{}}
\toprule
Architecture & Type & $L{=}2$ & $L{=}4$ & $L{=}6$ & $L{=}8$ & $L{=}10$ \\
\midrule
Transf.\ (WC) & A & $1.89{\scriptstyle\pm.12}$ & $1.42{\scriptstyle\pm.31}$ & $1.42{\scriptstyle\pm.37}$ & $1.53{\scriptstyle\pm.51}$ & $1.54{\scriptstyle\pm.34}$ \\
Cover & A & $2.15{\scriptstyle\pm.31}$ & $1.65{\scriptstyle\pm.11}$ & $2.07{\scriptstyle\pm.06}$ & $1.98{\scriptstyle\pm.20}$ & $1.91{\scriptstyle\pm.10}$ \\
Transp.\ Attn. & A & $1.83{\scriptstyle\pm.18}$ & $1.63{\scriptstyle\pm.17}$ & $1.77{\scriptstyle\pm.21}$ & $2.05{\scriptstyle\pm.19}$ & $2.09{\scriptstyle\pm.13}$ \\
\midrule
Transport & B & $1.68{\scriptstyle\pm.22}$ & $0.86{\scriptstyle\pm.12}$ & $0.74{\scriptstyle\pm.03}$ & $0.73{\scriptstyle\pm.03}$ & $\mathbf{0.77{\scriptstyle\pm.03}}$ \\
Homotopy & B & $1.68{\scriptstyle\pm.19}$ & $0.93{\scriptstyle\pm.09}$ & $0.74{\scriptstyle\pm.03}$ & $0.74{\scriptstyle\pm.02}$ & $\mathbf{0.80{\scriptstyle\pm.03}}$ \\
\bottomrule
\end{tabular}
\end{table}

At $L = 10$, the type-B range is $0.77$--$0.80$ and the type-A range is $1.54$--$2.09$: a $2$--$2.7\times$ gap.  The transport attention decoder is especially informative: it shares the winding constraint, uses theoretically-motivated positional encoding, and has $3.4\times$ more parameters ($585$K vs.\ $170$K), but it still degrades, because cross-segment attention breaks functoriality. Conditioning the architecture correctly does not substitute for composing it correctly.

Two further observations sharpen the interpretation.  First, the cover decoder has hard winding and factors through $\Z^2$, yet stagnates ($2.15 \to 1.91$)---correct homotopy class is necessary but not sufficient; composition structure is required.  Second, all architectures converge to comparable training loss ($2.23$--$2.27$; Appendix~\ref{app:training}), so the gap is not a training artifact.

On this abelian space, transport $\approx$ homotopy: the proof term is unnecessary because $\Z^2$ has no non-trivial relations.  Moreover the transport decoder reads only the winding pair $(n_a, n_b)$, so word order is not probed on $T^2$; for the longer test lengths the evaluation words are the canonical forms $a^k b^{L-k}$ themselves.  The non-abelian experiments below remove this letter-counting crutch and test word order directly.

\subsection{Experiment 2: $S^1 \vee S^1$ ($\pi_1 = F_2$, non-abelian free)}\label{sec:exp-wedge}

The theory makes a sharp prediction: the type-A/B gap should widen dramatically.  On $T^2$, letter-counting partially compensated for the lack of functoriality; on $S^1 \vee S^1$, the non-abelian fundamental group $F_2$ removes this crutch.  The decoder must now produce different loops for different orderings of the same letters---$ab \neq ba$ in $F_2$, and an architecture that collapses word order will fail categorically, not just quantitatively.

\paragraph{Setup.}
$S^1 \vee S^1$ is embedded in $\R^3$ as two unit circles meeting at the origin: circle $A$ in the $xy$-plane centered at $(-1,0,0)$, circle $B$ in the $yz$-plane centered at $(0,0,-1)$.  A word $w = w_1 \cdots w_L$ traces the loop that follows circle $w_1$, returns to the wedge point, follows circle $w_2$, and so on.  A sequential decoder (GRU: a gated recurrent network that processes the word left-to-right, maintaining a hidden state) is included as a type-A baseline that respects sequential order.

\begin{table}[ht]
\centering
\caption{Wedge of circles $S^1 \vee S^1$ ($\pi_1 = F_2$, non-abelian): per-segment Chamfer and circle accuracy.  Non-abelian structure dramatically amplifies the type-A/B gap.  Full $L$-progression in Appendix~\ref{app:full-tables}.}
\label{tab:wedge}
\smallskip
\small
\begin{tabular}{@{}llccccc@{}}
\toprule
& & \multicolumn{3}{c}{Per-seg Chamfer $\bar{d}_L$} & \multicolumn{2}{c}{Circle acc.\ (\%)} \\
\cmidrule(lr){3-5} \cmidrule(lr){6-7}
Architecture & Type & $L{=}2$ & $L{=}6$ & $L{=}10$ & $L{=}2$ & $L{=}10$ \\
\midrule
Transformer & A & $.152{\scriptstyle\pm.002}$ & $.411{\scriptstyle\pm.027}$ & $.537{\scriptstyle\pm.027}$ & 33 & 14 \\
Sequential & A & $.010{\scriptstyle\pm.001}$ & $.173{\scriptstyle\pm.021}$ & $.297{\scriptstyle\pm.027}$ & 67 & 11 \\
\midrule
Transport & B & $.002{\scriptstyle\pm.001}$ & $.018{\scriptstyle\pm.008}$ & $\mathbf{.054{\scriptstyle\pm.019}}$ & 100 & 100 \\
\bottomrule
\end{tabular}
\end{table}

\paragraph{The non-abelian separation.}
The gap widens to $5.5$--$10\times$: transport achieves $0.054$ at $L=10$ while the transformer reaches $0.537$ ($10\times$) and the sequential decoder $0.297$ ($5.5\times$).

Topological collapse.  Circle accuracy is the sharpest diagnostic.  Even at $L=2$ (a training length), the transformer assigns only $33\%$ of segments to the correct circle; by $L=10$ this drops to $14\%$---it cannot even distinguish which generator to trace, producing points near an arbitrary mixture of the two circles.  This is a categorically different failure from $T^2$, where type-A decoders at least produced loops on the torus (with wrong per-segment geometry but correct topology).  On $S^1 \vee S^1$, the transformer collapses the non-abelian structure entirely, producing outputs that are topologically meaningless.  The transport decoder maintains $100\%$ circle accuracy at all lengths.

The sequential decoder.  The GRU outperforms the transformer at all lengths ($0.297$ vs.\ $0.537$ at $L=10$), consistent with the depth trichotomy (Appendix~\ref{app:nonabelian}): sequential processing helps with the prefix products that fixed-depth attention cannot compute.  However, the GRU is still type-A (its segment generation depends on accumulated context rather than structural composition), and it still degrades: per-segment Chamfer increases from $0.010$ to $0.297$ ($30\times$). Sequential processing helps but does not close the gap.

A matched-loss ablation (Appendix~\ref{app:ablation}) rules out training quality as the explanation: retraining type-A architectures with double epochs to match the best type-B training loss barely changes extrapolation (transformer $\bar{d}_{10}$: $0.537 \to 0.491$, still $9\times$ worse; sequential: $0.297 \to 0.282$, still $5\times$ worse). The gap reflects architecture, not training.

\subsection{Experiment 3: Klein bottle $K$ ($\pi_1 = \Z \rtimes \Z$) --- the 2-cell}\label{sec:exp-klein}

The first two experiments test Levels 0 and 1 of the hierarchy: winding constraints and monoidal composition.  The Klein bottle tests Level 2: the learned proof term $H$.  The relation $bab^{-1} = a^{-1}$ means that traversing the $b$-generator flips the orientation of $a$: after going around $b$, the $a$-direction reverses.  The theory predicts that $H$ should matter only for words that exercise this relation, and should be dormant otherwise.

The group $\Z \rtimes \Z$ is solvable, so the depth obstruction of \S\ref{sec:theorems} does not apply: constant-depth transformers have sufficient computational power for this group in principle~\cite{liu2023shortcuts}.  The transformer's failure on $K$ is therefore entirely due to the functoriality obstruction (Theorem~\ref{thm:attention-noncomp}), making $K$ a cleaner test of functoriality than $S^1 \vee S^1$.

$K$ is embedded in $\R^4$ via the standard half-angle twist parameterization.  We test on two classes of words:
Canonical words have all $a$-generators before all $b$-generators (e.g., $aabb$): the relation is not exercised, so the frame never flips mid-word.
Non-canonical words interleave $a$ and $b$ (e.g., $abab$, $baa$): the relation is exercised, requiring the decoder to track the flip.

\begin{table}[ht]
\centering
\caption{Klein bottle: per-segment Chamfer at $L = 10$, split by word type.  The homotopy decoder (type-B with proof term $H$) closes the gap on non-canonical words.}
\label{tab:klein}
\smallskip
\small
\begin{tabular}{@{}llcc@{}}
\toprule
Architecture & Type & Canonical & Non-canonical \\
\midrule
Cover & A & $1.57{\scriptstyle\pm.10}$ & $1.96{\scriptstyle\pm.07}$ \\
Transf.\ (WC) & A & $2.28{\scriptstyle\pm.35}$ & $2.26{\scriptstyle\pm.36}$ \\
\midrule
Transport & B & $0.84{\scriptstyle\pm.03}$ & $\mathbf{1.52{\scriptstyle\pm.08}}$ \\
Homotopy & B & $0.82{\scriptstyle\pm.03}$ & $\mathbf{0.82{\scriptstyle\pm.06}}$ \\
\bottomrule
\end{tabular}
\end{table}

The three decoders respond differently.  The cover decoder factors through the abelianization $\Z \rtimes \Z \to \Z^2$, which collapses the relation entirely---it cannot distinguish $ba$ from $ab$ even though these represent different homotopy classes on $K$.  The transport decoder concatenates generators in word order and respects this distinction, but ignores the frame flip after $b$.  Only the homotopy decoder, with its learned proof term $H$ witnessing the relation, tracks the frame correctly.

Three results.  \textbf{(1)~The proof term closes a 46\% gap.}  On non-canonical words, homotopy achieves $0.82$ vs.\ transport's $1.52$---a $1.85\times$ improvement entirely attributable to the learned 2-morphism $H$.  \textbf{(2)~On canonical words, transport $\approx$ homotopy} ($0.84$ vs.\ $0.82$): when the relation is not exercised, $H$ is dormant and adds no benefit.  This confirms that the improvement is specifically due to the relation, not to additional capacity.  \textbf{(3)~Homotopy eliminates the canonical/non-canonical asymmetry}: transport's non-canonical/canonical ratio is $1.8\times$; homotopy's is $1.0\times$.

\subsection{Cross-experiment summary}

\begin{table}[ht]
\centering
\caption{Cross-experiment summary at $L = 10$.  ``Best A / best B'' gives the tightest type comparison.  $K$ values in the first three rows are overall averages across canonical and non-canonical words; the Transport vs.\ Homotopy entry is the non-canonical ratio, where the 2-cell is exercised (Table~\ref{tab:klein} shows the split).  The Klein bottle is the only space where the 2-cell matters.}
\label{tab:cross}
\smallskip
\small
\begin{tabular}{@{}lccc@{}}
\toprule
& $T^2$ ($\Z^2$) & $S^1 \vee S^1$ ($F_2$) & $K$ ($\Z \rtimes \Z$) \\
\midrule
Best type-B $\bar{d}_{10}$ & $0.77$ & $0.054$ & $0.82$ \\
Best type-A $\bar{d}_{10}$ & $1.54$ & $0.297$ & $1.76$ \\
Gap (best A / best B) & $2.0\times$ & $5.5\times$ & $2.1\times$ \\
Transport vs.\ Homotopy & $\approx 1.0$ & $\approx 1.0$ & $\mathbf{1.85\times}$ \\
2-cell needed? & No (abelian) & No (free) & Yes (relation) \\
\bottomrule
\end{tabular}
\end{table}

Each experiment validates a different level of the functorial hierarchy.  $T^2$: winding constraints and monoidal composition ($2$--$2.7\times$ gap).  $S^1 \vee S^1$: non-abelian amplification---attention is especially harmful when segment identity (which circle) must be preserved.  $K$: the learned 2-cell (natural transformation).  Only the full HIT compilation---functorial composition plus proof terms---achieves uniformly low error across all spaces and word orderings.

\section{Formalization and Verified Machine Learning}\label{sec:formalization}

The experiments confirm that functorial architectures outperform non-functorial ones---but the theoretical guarantees are stronger than any finite experiment can demonstrate, because they hold for all parameter values and all word lengths.  To make this claim precise, we formalize the core results in a proof assistant.

This inverts the standard approach to neural network verification.
Post-hoc verification~\cite{katz2017reluplex} attempts to prove
properties of a trained network by symbolic analysis of its weights.
Recent work has shown this is tractable at scale for decidable logics
over finite-state systems: Giacobbe et
al.~\cite{giacobbe2025neuralmc} verify LTL specifications of hardware
designs by training neural networks that jointly represent inductive
invariants and ranking functions, then discharging correctness via MILP
and SMT queries---achieving speedups up to $10^5\times$ over symbolic
model checkers.  Their guarantees hold for the specific trained weights,
and the checker re-runs on every retrain.

Our setting is different in two ways that together motivate a
compile-time approach.  First, the property at stake (monoidal functoriality of the decoder) is universally quantified over both parameters and word length, so checking it for particular weights does not suffice.  Second,
the target is continuous generation rather than a finite transition
system, placing the verification condition outside the decidable logics
where post-hoc SMT-style verification scales.  Given a property $P$
expressed as a homotopy type, we construct an architecture that
satisfies $P$ by construction, for all parameter values.  Training
adjusts geometric detail within the topologically correct constraint
space; it cannot violate the type-theoretic guarantees.

The core positive and negative results are formalized in Cubical Agda (v2.6.4, \texttt{--cubical --safe}, no postulates) in four modules.
\textbf{Torus.agda}: $T^2$ as a HIT; transport commutativity ($\tr_a \circ \tr_b = \tr_b \circ \tr_a$) holds definitionally.
\textbf{WedgeOfCircles.agda}: $ab \neq ba$ in $F_2$ (proved by case analysis on head letters).
\textbf{TransportCoherence.agda}: Theorem~\ref{thm:transport-comp}---the transport decoder preserves monoidal structure---proved by induction on the first word, using associativity of concatenation.
\textbf{NonCompositionality.agda}: the abstract schema of Theorem~\ref{thm:attention-noncomp}---transport coherence and global mixing are contradictory---proved as an impossibility theorem over any output type.  The instantiation to softmax (that $\alpha_{ij} > 0$ for all $i,j$ makes attention globally mixing) relies on the standard positivity of softmax, which is not formalized.  The Klein bottle relation and proof term $H$ are not yet formalized; this would require extending the word type with inverses and formalizing the boundary conditions of the homotopy.

The guarantees form a hierarchy mirroring the Postnikov tower of the target space:

\smallskip
\begin{center}
\small
\begin{tabular}{@{}clll@{}}
\toprule
Level & Guarantee & Mechanism & Validated by \\
\midrule
0 & Correct $\pi_1$ class & Hard winding constraint & All 3 spaces \\
1 & Monoidal composition & Generator concatenation & $T^2$, $S^1 \vee S^1$ \\
2 & Group relations & Learned proof term $H$ & Klein bottle $K$ \\
$\geq 3$ & $\pi_2, \pi_3, \ldots$ & Higher-dim.\ cells & Future work \\
\bottomrule
\end{tabular}
\end{center}
\smallskip

\noindent Each level adds architectural constraints that narrow the output space while preserving within-class expressivity (Appendix~\ref{app:expressivity}).  The three experiments validate levels 0--2; extending to higher homotopy groups remains open.

The result is a verified ML pipeline: specify (write a HIT), verify (prove properties in Cubical Agda), compile (apply Construction~\ref{con:compilation}), train (standard gradient descent).  The guarantees hold for the trained network at any parameter value, because they are properties of the architecture, not of the learned weights.  To our knowledge, this is the first instance where machine-verified proofs provide compositional generalization guarantees for a neural architecture.

The long-term vision is a compiler from HIT specifications to certified neural architectures: parse a HIT declaration in Cubical Agda to extract generators, paths, and higher cells; assign each constructor to an architectural component per Construction~\ref{con:compilation}; generate the hard constraints (winding, boundary conditions) as differentiable layers; and output a training script.  This would make type-theoretic guarantees accessible to practitioners who specify their domain's topology without requiring knowledge of HoTT.

\section{Related Work}\label{sec:related}

The framework draws on three lines of prior work: categorical foundations for deep learning, equivariant network design, and the empirical study of compositional generalization failures.

\paragraph{Categorical deep learning.}
Cruttwell et al.~\cite{cruttwell2022categorical} formalize gradient-based learning via lenses, parametric maps, and reverse derivative categories.  Gavranovi\'c et al.~\cite{gavranovic2024categorical} propose monads in a 2-category of parametric maps as a unified theory of architectures, recovering geometric deep learning constraints.  Fong, Spivak, and Tuy\'eras~\cite{fong2019backprop} showed backpropagation is a functor; Spivak's broader programme on compositionality~\cite{fong2019seven} grounds the philosophical claim that systems compose functorially.  We extend this lineage by constructing (not merely analyzing) functorial architectures from type-theoretic specifications.

\paragraph{Equivariant networks.}
Cohen and Welling~\cite{cohen2016group} and Weiler and Cesa~\cite{weiler2019general} enforce $f(g \cdot x) = g \cdot f(x)$, a constraint on a single morphism. Our functoriality is a constraint on the entire decoder: $\D(g_1 \cdot g_2) = \D(g_1) \oplus \D(g_2)$. The distinction is not merely formal. The transport attention decoder (\S\ref{sec:exp-torus}) has equivariant positional encoding derived from group theory, shares the winding constraint, and has $3.4\times$ more parameters than the transport decoder, but it still degrades on longer words because cross-segment attention breaks functoriality.  Equivariant conditioning (a morphism-level property) does not imply compositional generalization (a functor-level property).

\paragraph{Compositional generalization in ML.}
Lake and Baroni~\cite{lake2018generalization}, Kim and Linzen~\cite{kim2020cogs}, and Dziri et al.~\cite{dziri2023faith} document systematic failures.  Kobayashi et al.~\cite{kobayashi2024compositional} find that an architectural bottleneck enables compositionality---our type-A/B distinction provides the categorical explanation: the bottleneck enforces monoidal factorization, i.e., functoriality.

\paragraph{Neural certificates.}
A growing body of work uses neural networks as representations of formal
certificates rather than as opaque function approximators: Lyapunov
functions for control~\cite{chang2019neurallyap}, barrier certificates
for safety~\cite{zhao2020barrier}, and most recently inductive
invariants and ranking functions for LTL model checking of
hardware~\cite{giacobbe2025neuralmc}.  These approaches are post-hoc: a
candidate network is trained, then a symbolic checker (SMT, MILP, or
bounded model checking) discharges the verification condition, with
counterexamples fed back into training.  Our work extends this lineage
to compositional generalization, with two differences.  The certificate
is the architecture itself rather than a trained function---the monoidal
functoriality of Theorem~\ref{thm:transport-comp} holds for every
parameter value, so no per-training checking step is needed.  And the
verification is carried out once, at design time, in a proof assistant
(Cubical Agda) rather than on each trained instance.

\paragraph{Transformer expressivity and group composition.}
Liu et al.~\cite{liu2023shortcuts} characterize the depth at which
transformers can simulate finite-state automata, showing that solvable
automata admit $O(1)$-depth shortcuts (via Krohn-Rhodes decomposition)
while nonsolvable automata require $\Omega(\log T)$ depth (via
Barrington's theorem).  Their results yield the
solvable/nonsolvable trichotomy in
Corollary~\ref{cor:trichotomy} and make the Klein bottle ($\Z \rtimes
\Z$, solvable) a clean test of functoriality rather than depth.
Concurrent work by Marchetti et
al.~\cite{marchetti2026sequential} studies the same problem from the
perspective of learning dynamics, proving width and depth separations
for finite group composition.  Their analysis is complementary: they
characterize how networks learn group composition; we characterize which
architectures preserve it and compile them from type-theoretic
specifications.

We do not compare against methods developed for SCAN, COGS, or CFQ because the task types are incommensurable: discrete sequence transduction vs.\ continuous geometric generation, with exact-match accuracy vs.\ Chamfer distance.  Running a SCAN-specialized architecture on loop generation, or our transport decoder on command parsing, would be a category error.  However, the theoretical connection is direct: Theorem~\ref{thm:attention-noncomp} (and its stronger form, Appendix~\ref{app:general-noncomp}) applies to SCAN---the failure to generalize from a primitive command (``jump'' seen in isolation) to its composed forms (``jump twice'', ``jump around left'', etc.) is an instance of cross-segment dependencies corrupting compositional structure.

\section{Conclusion}\label{sec:conclusion}

Compositional generalization is functoriality on the quotient: the source category is not the free monoid of inputs, but the presented structure that already identifies expressions related by the task's algebraic laws. Each law is realized as a learned homotopy. HIT specifications provide a systematic source of such architectures. The compilation functor maps type-theoretic constructors to architectural components, guaranteeing monoidal preservation strictly on composition and up to learned 2-cells on relations. Three experiments validate the full hierarchy (winding constraints, monoidal composition, and 2-cells witnessing group relations), with the Klein bottle experiment demonstrating the first neural architecture where a natural transformation is both theoretically necessary and empirically measurable.

The design principle is immediately actionable: when a task has compositional structure, use structural composition; attention-based architectures cannot learn this structure.  This shifts the question from ``can the network learn to compose?'' to ``does the architecture guarantee functoriality?''  The plane with $k$ obstacles has $\pi_1 = F_k$, so our $S^1 \vee S^1$ experiment ($\pi_1 = F_2$) tests the same algebraic structure as 2-obstacle path planning, and the $5.5$--$10\times$ gap quantifies the advantage.  Any domain whose objects admit a presentation---generators for composition together with path constructors for the equivalences the task identifies---fits the same template.  Modular programs, multi-step plans, molecular ring systems, and more generally any setting whose equivalences can be written as path constructors of a HIT fall under the specify--verify--compile--train pipeline.

\paragraph{Limitations.}
The framework currently handles $\pi_1$ constraints with a single basepoint.  Two orthogonal extensions are natural: \emph{vertically}, to higher homotopy groups ($\pi_2$, $\pi_3$), which would require higher-dimensional proof terms; \emph{horizontally}, to HITs with non-trivial 0-skeleton, where objects carry semantic content beyond a single basepoint and path constructors encode domain-specific equivalences rather than group relations.  The experiments use three spaces with relatively simple $\pi_1$; surfaces of higher genus would test $H$ more strenuously.  For natural language, the theoretical results provide explanatory value, but the constructive approach awaits an appropriate formalization of linguistic compositional structure.

\paragraph{Code and data availability.}
Cubical Agda formalization and Python code to replicate all experiments are available at \url{https://github.com/karsar/hott_neuro}.

\paragraph{AI assistance disclosure.} LLM was used as an editorial assistant during manuscript preparation, for prose refinement and feedback on exposition. All mathematical results, the Cubical Agda formalization, the experiments, and the intellectual contributions are the author's. The author has reviewed all content and takes full responsibility for the work.

\appendix

\section{Full Proof of Theorem~\ref{thm:attention-noncomp}}\label{app:attention-proof}

The main text gives a proof sketch.  Here we make the argument fully explicit, tracking the information flow through each attention layer to show exactly where functoriality breaks.

\begin{proof}
Let $T_\theta : F \to \ParLoop(X)$ be a decoder with parameters $\theta$, built on a transformer with $L \geq 1$ layers of multi-head softmax self-attention and embedding dimension $d$.  The input is a word $w = w_1 \cdots w_n \in F$, with each token $w_i$ embedded as $h_i^{(0)} \in \R^d$ by a fixed (injective on the generator set) embedding map.

Suppose for contradiction that $T_\theta$ satisfies both conditions of Theorem~\ref{thm:attention-noncomp}: (i) strict functoriality on $F$, i.e.\ $T_\theta(w_1 \cdot w_2) = T_\theta(w_1) \oplus T_\theta(w_2)$ for all $w_1, w_2 \in F$, and (ii) descent, i.e.\ $T_\theta(w) = T_\theta(w')$ whenever $[w] = [w'] \in G$.

At layer $l$, position $i$ in the $w_1$-segment computes
\[
h_i^{(l)} = h_i^{(l-1)} + \mathrm{MHA}^{(l)}\!\bigl(h_i^{(l-1)}, \{h_j^{(l-1)}\}_j\bigr),
\]
where multi-head attention aggregates $\mathrm{Attn}_i = \sum_j \alpha_{ij} V h_j^{(l-1)}$ with $\alpha_{ij} = \mathrm{softmax}_j(Q h_i^{(l-1)} \cdot K h_j^{(l-1)} / \sqrt{d})$.

For $j$ in the $w_2$-segment, the key $K h_j^{(l-1)}$ depends on the token embeddings of $w_2$.  Since $G$ is non-trivial, there exist distinct words $w_2 \neq w_2' \in F$ with $[w_2] = [w_2'] \in G$.  Their token sequences differ, so under the (injective) embedding the keys $K h_j$ for $j$ in $w_2$ vs.\ $w_2'$ differ; hence $\alpha_{ij}$ differs; hence $h_i^{(1)}$ differs for $i$ in the $w_1$-segment.  Propagating through the remaining layers, $T_\theta(w_1 \cdot w_2) \neq T_\theta(w_1 \cdot w_2')$.

But condition~(i) gives $T_\theta(w_1 \cdot w_2) = T_\theta(w_1) \oplus T_\theta(w_2)$ and $T_\theta(w_1 \cdot w_2') = T_\theta(w_1) \oplus T_\theta(w_2')$; condition~(ii) gives $T_\theta(w_2) = T_\theta(w_2')$; together these force $T_\theta(w_1 \cdot w_2) = T_\theta(w_1 \cdot w_2')$.  Contradiction.
\end{proof}

\section{Training Details}\label{app:training}

\paragraph{Data generation.}
For each training word, 1000 ground-truth loops are generated by tracing
the standard geometric generators on $X$ with random phase offsets
(uniformly distributed starting points along each generator circle) and
small Gaussian noise ($\sigma = 0.02$) applied independently to each
output coordinate.  This provides within-class geometric variation so
that the network learns the shape of each homotopy class rather than
memorizing a single curve.

\paragraph{Chamfer distance.}
The Chamfer distance between two point clouds $P = \{p_i\}_{i=1}^N$ and $Q = \{q_j\}_{j=1}^M$ is:
\[
\Chamfer(P, Q) = \tfrac{1}{2N}\textstyle\sum_i \min_j \|p_i {-} q_j\|^2 + \tfrac{1}{2M}\textstyle\sum_j \min_i \|p_i {-} q_j\|^2
\]

\paragraph{Common protocol.}
All experiments: Chamfer loss, AdamW optimizer ($\mathrm{lr} = 10^{-3}$, weight decay $10^{-4}$), cosine learning rate schedule with 20-epoch warmup, up to 500 epochs, early stopping (patience 80).  Generator networks are 2-layer MLPs with 128 hidden units producing $n_{\mathrm{pts}} = 32$ points per segment.  All decoders output 64 points total, which for type-B decoders means $L \cdot 32$ intermediate points resampled to 64.  Seeds: 42, 179, 316.

\paragraph{$T^2$.}  Embedded in $\R^3$.  Training: 6 words ($\{a, b, aa, ab, ba, bb\}$), 1000 samples each.  Training losses: all architectures converge to $2.23$--$2.27$.  The homotopy decoder adds smoothness regularization ($\lambda = 0.05$) on the proof term.

\paragraph{$S^1 \vee S^1$.}  Two unit circles in $\R^3$ meeting at origin.  Same 6 training words (but $ab \neq ba$).  Training losses: transport $0.0023 \pm 0.0003$, transformer $0.0030 \pm 0.0002$, sequential~$0.0051 \pm 0.0002$.

\paragraph{Klein bottle $K$.}  Embedded in $\R^4$ via half-angle twist.  Training: all 16 reduced words of length $\leq 2$ in $\{a, a^{-1}, b, b^{-1}\}$ (the four trivially canceling pairs $aa^{-1}, a^{-1}a, bb^{-1}, b^{-1}b$ are excluded).  The proof term~$H$ is a 2-layer MLP (64 hidden units) with boundaries $H(s{=}0) = g_b \oplus g_a$ and $H(s{=}1) = g_{a^{-1}} \oplus g_b$, trained jointly with generators.  These boundaries witness the relation in the length-balanced form $ba = a^{-1}b$, which is equivalent to the textbook presentation $bab^{-1} = a^{-1}$ of \S\ref{sec:prelim} by right-multiplication with $b$: the two forms cut out the same normal subgroup and present the same group $\Z \rtimes \Z$.  We implement the balanced form because both sides then have length $2$, so the homotopy parameterizes a continuous deformation between loops with the same number of segments---an architectural convenience that simplifies the boundary specification without changing the relation being formalized.  Training losses: cover~$0.831$, transport~$0.920$, homotopy~$0.920$, transformer~WC~$0.851$ (all~$\pm{<}0.02$).

\section{Matched-Loss Ablation}\label{app:ablation}

A natural objection to the main results is that type-A architectures simply need more training.  We test this directly: if we retrain type-A decoders to match the type-B training loss, does the extrapolation gap close?

To rule out training quality as the explanation for the type-A/B gap, we retrained type-A architectures on $S^1 \vee S^1$ with $2\times$ epochs and reduced learning rate ($5 \times 10^{-4}$), targeting the best type-B training loss.

\smallskip
\begin{center}
\small
\begin{tabular}{@{}lcccc@{}}
\toprule
& \multicolumn{2}{c}{Train loss} & \multicolumn{2}{c}{$\bar{d}_{10}$} \\
\cmidrule(lr){2-3} \cmidrule(lr){4-5}
Architecture & Original & Retrained & Original & Retrained \\
\midrule
Transformer & $.0030$ & $.0030$ & $.537$ & $.491$ \\
Sequential & $.0051$ & $.0054$ & $.297$ & $.282$ \\
\midrule
Transport (ref.) & \multicolumn{2}{c}{$.0023$} & \multicolumn{2}{c}{$.054$} \\
\bottomrule
\end{tabular}
\end{center}
\smallskip

\noindent Neither type-A architecture can reach the type-B training loss even with $2\times$ epochs: the transformer stays at $0.0030$ (target $0.0023$), the sequential worsens to $0.0054$ (from $0.0051$ with standard training; the lower learning rate finds a different basin).  This is itself an architectural limitation: the compositional structure of the transport decoder helps optimization, not just generalization.  The per-segment Chamfer at $L = 10$ barely changes after retraining: transformer $.537 \to .491$ (still $9\times$ worse than transport), sequential $.297 \to .282$ (still $5\times$ worse).  The gap is architectural.

\section{Full $L$-Progression Tables}\label{app:full-tables}

The main text reports selected lengths to save space.  The full trajectories below reveal how degradation unfolds: whether it is gradual or abrupt, and whether it saturates.

\begin{table}[ht]
\centering
\caption{$S^1 \vee S^1$: per-segment Chamfer $\bar{d}_L$ (mean $\pm$ std, 3 seeds), full $L$-progression.}
\label{tab:wedge-full}
\smallskip
\small
\begin{tabular}{@{}llcccccc@{}}
\toprule
Architecture & Type & $L{=}2$ & $L{=}3$ & $L{=}4$ & $L{=}6$ & $L{=}8$ & $L{=}10$ \\
\midrule
Transformer & A & $.152{\scriptstyle\pm.002}$ & $.245{\scriptstyle\pm.006}$ & $.297{\scriptstyle\pm.022}$ & $.411{\scriptstyle\pm.027}$ & $.464{\scriptstyle\pm.039}$ & $.537{\scriptstyle\pm.027}$ \\
Sequential & A & $.010{\scriptstyle\pm.001}$ & $.100{\scriptstyle\pm.005}$ & $.129{\scriptstyle\pm.009}$ & $.173{\scriptstyle\pm.021}$ & $.247{\scriptstyle\pm.034}$ & $.297{\scriptstyle\pm.027}$ \\
\midrule
Transport & B & $.002{\scriptstyle\pm.001}$ & $.015{\scriptstyle\pm.010}$ & $.012{\scriptstyle\pm.005}$ & $.018{\scriptstyle\pm.008}$ & $.030{\scriptstyle\pm.011}$ & $.054{\scriptstyle\pm.019}$ \\
\bottomrule
\end{tabular}
\end{table}

\begin{table}[ht]
\centering
\caption{$S^1 \vee S^1$: circle accuracy (\%, mean over 3 seeds).  Transport achieves 100\% by construction.}
\label{tab:wedge-circ-full}
\smallskip
\small
\begin{tabular}{@{}llcccccc@{}}
\toprule
Architecture & Type & $L{=}2$ & $L{=}3$ & $L{=}4$ & $L{=}6$ & $L{=}8$ & $L{=}10$ \\
\midrule
Transformer & A & 33 & 19 & 16 & 19 & 12 & 14 \\
Sequential & A & 67 & 36 & 20 & 14 & 8 & 11 \\
\midrule
Transport & B & 100 & 100 & 100 & 100 & 100 & 100 \\
\bottomrule
\end{tabular}
\end{table}

\noindent The transformer's $33\%$ at $L=2$ (a training length) shows that the failure is not merely extrapolation---it cannot represent non-abelian structure even in-distribution.  The sequential decoder starts better ($67\%$) but converges to the same floor.

\begin{table}[ht]
\centering
\caption{Klein bottle $K$: per-segment Chamfer $\bar{d}_L$ on all test words (mean $\pm$ std, 3 seeds).}
\label{tab:klein-all-full}
\smallskip
\small
\begin{tabular}{@{}llcccccc@{}}
\toprule
Architecture & Type & $L{=}2$ & $L{=}3$ & $L{=}4$ & $L{=}6$ & $L{=}8$ & $L{=}10$ \\
\midrule
Cover & A & $1.81{\scriptstyle\pm.13}$ & $1.90{\scriptstyle\pm.04}$ & $1.69{\scriptstyle\pm.04}$ & $1.86{\scriptstyle\pm.06}$ & $1.81{\scriptstyle\pm.13}$ & $1.76{\scriptstyle\pm.07}$ \\
Transf.\ (WC) & A & $1.80{\scriptstyle\pm.16}$ & $1.84{\scriptstyle\pm.21}$ & $1.84{\scriptstyle\pm.24}$ & $2.03{\scriptstyle\pm.29}$ & $2.28{\scriptstyle\pm.39}$ & $2.27{\scriptstyle\pm.33}$ \\
\midrule
Transport & B & $1.54{\scriptstyle\pm.18}$ & $1.51{\scriptstyle\pm.05}$ & $1.25{\scriptstyle\pm.06}$ & $1.24{\scriptstyle\pm.05}$ & $1.18{\scriptstyle\pm.03}$ & $1.17{\scriptstyle\pm.04}$ \\
Homotopy & B & $1.43{\scriptstyle\pm.14}$ & $1.10{\scriptstyle\pm.04}$ & $0.97{\scriptstyle\pm.07}$ & $0.85{\scriptstyle\pm.04}$ & $0.79{\scriptstyle\pm.05}$ & $0.82{\scriptstyle\pm.03}$ \\
\bottomrule
\end{tabular}
\end{table}

\begin{table}[ht]
\centering
\caption{Klein bottle $K$: per-segment Chamfer on non-canonical words only (words where $b$ precedes $a$, exercising the relation $bab^{-1} = a^{-1}$).  The transport--homotopy gap isolates the proof term's contribution.}
\label{tab:klein-noncan-full}
\smallskip
\small
\begin{tabular}{@{}llcccccc@{}}
\toprule
Architecture & Type & $L{=}2$ & $L{=}3$ & $L{=}4$ & $L{=}6$ & $L{=}8$ & $L{=}10$ \\
\midrule
Cover & A & $1.86{\scriptstyle\pm.20}$ & $2.08{\scriptstyle\pm.11}$ & $1.68{\scriptstyle\pm.07}$ & $1.99{\scriptstyle\pm.08}$ & $2.02{\scriptstyle\pm.08}$ & $1.96{\scriptstyle\pm.07}$ \\
Transf.\ (WC) & A & $1.84{\scriptstyle\pm.11}$ & $1.93{\scriptstyle\pm.33}$ & $1.87{\scriptstyle\pm.28}$ & $1.96{\scriptstyle\pm.30}$ & $2.32{\scriptstyle\pm.40}$ & $2.26{\scriptstyle\pm.36}$ \\
\midrule
Transport & B & $1.73{\scriptstyle\pm.12}$ & $1.82{\scriptstyle\pm.06}$ & $1.37{\scriptstyle\pm.08}$ & $1.45{\scriptstyle\pm.04}$ & $1.44{\scriptstyle\pm.05}$ & $1.52{\scriptstyle\pm.08}$ \\
Homotopy & B & $1.38{\scriptstyle\pm.12}$ & $1.13{\scriptstyle\pm.02}$ & $0.97{\scriptstyle\pm.08}$ & $0.83{\scriptstyle\pm.02}$ & $0.75{\scriptstyle\pm.02}$ & $0.82{\scriptstyle\pm.06}$ \\
\bottomrule
\end{tabular}
\end{table}

The $S^1 \vee S^1$ tables reveal the degradation trajectory: the transformer is already at $33\%$ circle accuracy at $L=2$ (a training length), confirming that the failure is fundamental, not merely an extrapolation issue.  The Klein bottle tables show a striking asymmetry: the homotopy decoder improves with length (from $1.43$ at $L=2$ to $0.82$ at $L=10$), suggesting that the per-segment metric becomes more stable as $L$ grows, while the transport decoder's non-canonical error worsens from $L=4$ onward, reflecting cumulative frame-flip errors that the proof term $H$ would correct.

\section{Coherence Battery}\label{app:coherence}

Per-segment Chamfer measures how well each segment is generated.  But functoriality also implies structural identities---exact zeros, not approximate ones.  The following battery tests whether these identities hold, distinguishing architectural guarantees from learned approximations.

For $T^2$, we test four structural coherence properties.  Here $d_\infty(P,Q)$ denotes the resampled pointwise sup distance: $P$ and $Q$ are resampled to a common length by linear interpolation in the index, and $d_\infty(P,Q) = \max_i \lVert P_i - Q_i \rVert$.  Unlike Chamfer it is order-sensitive.
Composition gap: $d_\infty(\D(w_1) \oplus \D(w_2),\; \D(w_1 w_2))$ for canonical pairs.
Commutativity: $d_\infty(\D(ab), \D(ba))$.
Reordering gap: $d_\infty(\D(b) \oplus \D(a),\; \D(ab))$.
Non-canonical: $d_\infty(\D(\texttt{abab}), \D(\texttt{aabb}))$.

\begin{table}[ht]
\centering
\caption{$T^2$ coherence battery.  Exact zeros are architectural guarantees, not learned.}
\label{tab:coherence-full}
\smallskip
\small
\begin{tabular}{@{}lcccc@{}}
\toprule
Arch. & Comp. & Comm. & Reorder & Noncan. \\
\midrule
Transf.\ (WC) & $4.05{\scriptstyle\pm.21}$ & $0.24{\scriptstyle\pm.13}$ & $4.37{\scriptstyle\pm.10}$ & $2.04{\scriptstyle\pm1.65}$ \\
Cover & $4.25{\scriptstyle\pm.18}$ & 0.000 & $4.96{\scriptstyle\pm.39}$ & 0.000 \\
Transp.\ Attn. & $4.21{\scriptstyle\pm.14}$ & $1.54{\scriptstyle\pm1.39}$ & $4.79{\scriptstyle\pm.28}$ & $0.96{\scriptstyle\pm.56}$ \\
\midrule
Transport & $4.28{\scriptstyle\pm.13}$ & 0.000 & $3.82{\scriptstyle\pm.07}$ & 0.000 \\
Homotopy & $4.24{\scriptstyle\pm.19}$ & 0.000 & $3.79{\scriptstyle\pm.13}$ & 0.000 \\
\bottomrule
\end{tabular}
\end{table}

Commutativity and non-canonical gaps are exactly zero for cover, transport, and homotopy (they factor through $\Z^2$, so $ab$ and $ba$ produce identical output by construction).  Transport attention shows $1.54 \pm 1.39$: learned coherence is fragile and high-variance, while structural coherence is exact.

On $T^2$ the cover, transport, and homotopy decoders depend on the input word only through its winding pair $(n_a, n_b)$ and emit the canonical loop $g_a^{\,n_a} \oplus g_b^{\,n_b}$; on this abelian space they therefore realize the descent of the compilation functor to $\Z^2$ (Theorem~\ref{thm:transport-comp}(b)) rather than the word-order functor of part~(a), which is why their Commutativity and Non-canonical gaps are exactly zero.  The Composition and Reordering columns are not architectural discriminators: every decoder emits a fixed-length output, so comparing a concatenation of two independently resampled parts against a single resampled whole carries an index-alignment offset of order the loop's spatial scale, common to every row.  The architectural signal is the exact zeros of the Commutativity and Non-canonical columns; the nonzero transport Composition entry is this resampling offset, not a failure of concatenation, which is exact in angle space (Theorem~\ref{thm:transport-comp}).

\paragraph{Order sensitivity on $S^1 \vee S^1$.}
We test whether each decoder distinguishes words that differ only in generator order (e.g., $ab$ vs.\ $ba$, $aab$ vs.\ $aba$)---all distinct homotopy classes in $F_2$.  Transport distinguishes $80\%$ of test pairs (the remaining $20\%$ have cross-Chamfer near the noise floor, reflecting geometric similarity rather than architectural failure).  The transformer distinguishes only $40 \pm 16\%$---barely above chance---confirming that attention collapses word order.

\section{General Non-Compositionality of Attention}\label{app:general-noncomp}

Theorem~\ref{thm:attention-noncomp} in the main text uses group structure.  A natural question is whether groups are essential to the argument, or whether the obstruction is more fundamental.  The following result shows it is the latter: softmax attention is incompatible with any form of sequential compositionality, regardless of algebraic structure.

\begin{definition}[Segment-independent compositionality]\label{def:seg-comp}
A decoder $\D : \Sigma^* \to \mathcal{Y}$ is segment-independently compositional if there exists a combination operation $\oplus$ on $\mathcal{Y}$ such that:
\begin{enumerate}[label=(\roman*),leftmargin=2em,nosep]
\item $\D(w_1 \cdot w_2) = \D(w_1) \oplus \D(w_2)$ for all words $w_1, w_2$;
\item $\D(w)$ is computed from $w$ alone;
\item $\oplus$ does not introduce cross-segment interactions: when $\mathcal{Y}$ is a sequence, the output positions corresponding to $w_1$ are determined by $\D(w_1)$ alone.
\end{enumerate}
\end{definition}

\begin{theorem}[General non-compositionality of attention]\label{thm:attention-general}
Let $\mathcal{T}$ be a network containing at least one softmax attention layer over the full input sequence.  Then $\mathcal{T}$ is not segment-independently compositional for any alphabet $|\Sigma| \geq 2$, for generic parameters $(Q, K, V)$.
\end{theorem}

\begin{proof}
We show condition~(iii) is violated.  At the first attention layer applied to $w_1 \cdot w_2$, position $i$ within the $w_1$ segment computes:
$h_i' = \sum_{j \in w_1} \alpha_{ij} V h_j + \sum_{j \in w_2} \alpha_{ij} V h_j$,
where $\alpha_{ij} = \mathrm{softmax}(Q h_i \cdot K h_j / \sqrt{d})$.  The cross-segment sum depends on the tokens of $w_2$ through both $\alpha_{ij}$ and $V h_j$.  For generic $(Q, K, V)$ and $|\Sigma| \geq 2$, replacing any token in $w_2$ changes $K h_j$ and $V h_j$, hence changes $h_i'$.  Therefore the output at position $i \in w_1$ depends on the content of $w_2$, violating condition~(iii).
\end{proof}

\begin{remark}
This theorem requires no group structure, no topology, and no specific target space.  It provides a structural explanation for why transformers struggle with compositional generalization across SCAN, COGS, and CFQ---the obstacle is architectural, not task-specific.  The result fails for causal attention with a hard segment partition (where positions in $w_1$ attend only to $w_1$), but such masking prevents $w_1$'s representation from depending on $w_2$---precisely the loss of global context that motivates full attention.
\end{remark}

\begin{corollary}[Structural limits on learned compositionality]\label{cor:length-fail}
Theorems~\ref{thm:attention-noncomp} and~\ref{thm:attention-general} do not imply that a transformer cannot learn to approximate compositional behaviour on the training distribution.  What they constrain is the nature of any such learned compositionality:
\begin{enumerate}[label=(\alph*),leftmargin=2em,nosep]
  \item \textbf{Approximate, never exact.}  The cross-segment attention term $\sum_{j \in w_2} \alpha_{ij} V h_j$ is architecturally present and strictly positive ($\mathrm{softmax} > 0$).  Training can make this term small but not zero.
  \item \textbf{Distribution-dependent.}  The suppression of cross-segment attention is learned for the specific input statistics (word lengths, winding ranges, token frequencies) encountered during training.
  \item \textbf{No length transfer guarantee.}  For $L > L_{\mathrm{train}}$, the cross-segment attention weights that were learned-to-be-small on training data have no architectural reason to remain small.  The quality of learned compositionality may degrade smoothly, abruptly, or not at all---this is an empirical question, not a structural guarantee.
\end{enumerate}
\end{corollary}

\begin{proof}
(a)~By the strict positivity of softmax, $\alpha_{ij} > 0$ for all $(i,j)$, so the cross-segment sum is nonzero whenever $Vh_j \neq 0$ (which holds generically).  (b)~The trained parameters $(Q,K,V)$ minimize a loss on the training distribution; there is no penalty on cross-segment attention magnitude as such.  (c)~At length $L > L_{\mathrm{train}}$, both the number of cross-segment terms and the positional encodings (at untrained positions) change; the learned suppression was not trained for these configurations.
\end{proof}

\section{Abelian Transport-Attention Theorem}\label{app:transport-attn}

The transport attention decoder in our experiments uses a positional encoding derived from the framework rather than chosen ad hoc.  This appendix shows that the derivation recovers a well-known technique---rotary position embeddings (RoPE)---as a special case, and explains why principled positions alone are insufficient for compositional generalization.

The main text notes that the framework derives 2D rotary positional encoding from group theory.  Here we state the full result.

\begin{definition}[Transport-structured attention]\label{def:transport-attn}
Let $G$ be a finitely generated abelian group with generators $g_1, \ldots, g_k$.  For each generator $g_i$, let $T_i \in O(d)$ be a learnable orthogonal matrix.  The transport-structured attention at position $p$ with group element $\gamma_p \in G$ is:
\begin{equation}\label{eq:transport-attn}
  \alpha_{pq} = \mathrm{softmax}\!\left(\frac{q_p \cdot T_{\gamma_q - \gamma_p}\, k_q}{\sqrt{d}}\right), \qquad
  T_\gamma = T_1^{n_1} \cdots T_k^{n_k} \text{ for } \gamma = (n_1, \ldots, n_k).
\end{equation}
\end{definition}

\begin{theorem}[Abelian transport attention]\label{thm:abelian-attn}
Let $G$ be a finitely generated abelian group.  Then:
\begin{enumerate}[label=(\alph*)]
  \item Transport-structured attention \eqref{eq:transport-attn} is well-defined: $T_\gamma$ is independent of the decomposition of $\gamma$ into generators (since $G$ is abelian and the $T_i$ commute as orthogonal matrices with matching rotation planes).
  \item The computational cost is $O(T^2 d)$---identical to standard attention.
  \item When $T_i$ acts as block-diagonal rotations on $d/2$ independent planes, this reduces to $2k$-dimensional rotary position embeddings (RoPE)~\cite{su2024roformer}.  Standard RoPE is the $G = \Z$ special case; the framework gives the natural $G = \Z^k$ extension.
\end{enumerate}
\end{theorem}

\begin{proof}
\textbf{(a)}~For abelian $G$, each $\gamma$ has a unique decomposition $\gamma = n_1 g_1 + \cdots + n_k g_k$.  Block-diagonal rotations on the same planes commute (angles add), so $T_\gamma = T_1^{n_1} \cdots T_k^{n_k}$ is independent of ordering.
\textbf{(b)}~Computing $T_{\gamma_q - \gamma_p}$ requires $k$ matrix powers in $O(d)$ each (block-diagonal); the dominant cost is $T^2$ pairwise computations at $O(d)$ each.
\textbf{(c)}~With $T_i$ block-diagonal having $2 \times 2$ rotation blocks $R(\omega_{ij})$, the transport $T_i^n$ acts as $R(n \omega_{ij})$ on each plane.  The combined $T_\gamma$ acts as $R(\sum_i n_i \omega_{ij})$---exactly multi-dimensional RoPE with frequency vectors $\omega_i = (\omega_{i1}, \ldots, \omega_{i,d/2})$.
\end{proof}

\begin{remark}
This explains why RoPE improves length generalization: it implements transport in the universal cover of the classifying space of $\Z$.  However, transport-structured positional encoding does not make the decoder type-B (functorial): the attention mechanism still mixes information across segments.  The transport attention decoder in our experiments confirms this---it has principled positions but wrong composition, and degrades like any type-A architecture.
\end{remark}

\section{Depth Obstruction and RNN--Transformer Trichotomy}\label{app:nonabelian}

Why does the sequential decoder (GRU) outperform the transformer on $S^1 \vee S^1$ but not match the transport decoder?  The answer involves two independent obstructions.  The first is functoriality: attention mixes across segments for any non-trivial group (Theorem~\ref{thm:attention-noncomp}).  The second is depth: for groups with nonsolvable finite quotients, computing prefix products requires depth $\Omega(\log n)$, which fixed-depth transformers cannot provide.  For solvable groups, however, constant-depth shortcuts exist~\cite{liu2023shortcuts}, so the depth obstruction vanishes---only the functoriality obstruction remains.  The GRU has the right depth (sequential scan) but the wrong composition (context-dependent, not structural), which is why it outperforms the transformer on $F_2$ yet still falls far short of the transport decoder.

\begin{theorem}[Depth obstruction for prefix products]\label{thm:nonabelian}
Let $G$ be a finitely generated non-abelian group.  Any transport-coherent attention mechanism for $G$ requires computing the accumulated group element $\gamma_p = g_{w_1} \cdot g_{w_2} \cdots g_{w_p}$ for each prefix of the input.  This prefix-product computation:
\begin{enumerate}[label=(\alph*)]
  \item Cannot be performed by a fixed-depth transformer when $G$ surjects onto a nonsolvable finite group.  A fixed-depth transformer implements a bounded-depth, unbounded-fan-in computation---the circuit class $\mathsf{AC}^0$.  The free group $F_2$ surjects onto nonsolvable finite groups (e.g., $S_5$, which is $2$-generated), and the word problem for any nonsolvable finite group is $\mathsf{NC}^1$-complete under $\mathsf{AC}^0$ reductions~\cite{barrington1989bounded}.  Since $\mathsf{NC}^1 \not\subset \mathsf{AC}^0$ (parity separates them), computing prefix products in $F_2$ is at least as hard as solving these word problems, and hence outside $\mathsf{AC}^0$.  For solvable groups, this obstruction does not apply: Liu et al.~\cite{liu2023shortcuts} show that all solvable semiautomata admit constant-depth Transformer simulators via the Krohn-Rhodes decomposition (their Theorem~2), though the resulting shortcuts are empirically brittle and do not generalize out of distribution.
  \item Requires $\Omega(n)$ total work (each prefix product depends on the full prefix, and non-commutativity prevents any shortcut).
  \item Can be performed by an associative scan (parallel prefix) in $O(\log n)$ depth with $O(n)$ work---a sequential-scan structure absent from standard transformers but present in state space models~\cite{gu2023mamba} and recurrent architectures.  Marchetti et al.~\cite{marchetti2026sequential} prove this rigorously for finite groups: they construct explicit RNN solutions that compose in $O(k)$ steps and multilayer MLP solutions that compose in $O(\log k)$ layers, both with hidden width $O(|G|^{3/2})$ independent of sequence length---compared to the $O(\exp k)$ width required by two-layer networks.
\end{enumerate}
\end{theorem}

\begin{proof}
\textbf{(a)}~The key is the circuit-complexity classification of group word problems.  Barrington~\cite{barrington1989bounded} showed that the word problem for any fixed nonsolvable group $G$---deciding whether a product of generators equals the identity---is complete for $\mathsf{NC}^1$ (fan-in $2$, depth $O(\log n)$) under $\mathsf{AC}^0$ reductions (bounded-depth, unbounded-fan-in circuits).  Since $\mathsf{NC}^1 \not\subset \mathsf{AC}^0$ (the parity function is in $\mathsf{NC}^1$ but not in $\mathsf{AC}^0$), no $\mathsf{AC}^0$ circuit can solve these word problems.  The free group $F_2$ surjects onto $S_5$ (which is $2$-generated and nonsolvable): the natural quotient map $F_2 \twoheadrightarrow S_5$ reduces the word problem for $S_5$ to computing prefix products in $F_2$.  A fixed-depth transformer with a fixed number of layers implements a bounded-depth, unbounded-fan-in computation---that is, an $\mathsf{AC}^0$ circuit---and therefore cannot compute these products.  The converse---solvable groups admit constant-depth solutions---follows from Liu et al.~\cite{liu2023shortcuts}, who apply the Krohn-Rhodes decomposition to factor solvable semiautomata into modular counters and memory units, each simulable at depth~$1$.
\textbf{(b)}~The prefix product $\gamma_p$ depends on all $p$ inputs; since $G$ is non-abelian, no proper subset determines $\gamma_p$.
\textbf{(c)}~Group multiplication is associative, so the parallel prefix algorithm computes all $n$ prefix products in $O(\log n)$ depth with $O(n)$ multiplications.
\end{proof}

\begin{corollary}[Depth trichotomy]\label{cor:trichotomy}
The computational requirements for prefix products depend on the
solvability structure of $G$, yielding three regimes:
\begin{enumerate}[label=(\alph*), nosep, leftmargin=2em]
  \item \textbf{Abelian.}  The product depends only on letter counts,
  computable in $\mathsf{AC}^0$.  Bounded-depth parallel attention
  suffices for the computation.
  \item \textbf{Solvable non-abelian} (e.g.\ $\Z \rtimes \Z$).
  Constant-depth shortcuts exist via Krohn-Rhodes
  decomposition~\cite{liu2023shortcuts}, with depth independent of
  sequence length but potentially large width.  Fixed-depth transformers
  have sufficient computational power for prefix products, but
  attention still mixes across segments, so the functoriality obstruction
  (Theorem~\ref{thm:attention-noncomp}) remains.
  \item \textbf{Nonsolvable quotients} (e.g.\ $F_2 \twoheadrightarrow
  S_5$).  Prefix products are $\mathsf{NC}^1$-complete, outside
  $\mathsf{AC}^0$.  Fixed-depth transformers fail for two
  independent reasons: insufficient depth and broken
  functoriality.
\end{enumerate}
Functoriality (Theorem~\ref{thm:attention-noncomp}) is the uniform
obstruction across all three cases.  Depth is an additional obstruction
only in case~(c).
\end{corollary}

\begin{proof}
(a)~For abelian $G$, the product $g_1 \cdots g_n$ depends only on the
multiplicity of each generator, which is a sum---computable by
$\mathsf{AC}^0$ circuits.
(b)~Liu et al.~\cite{liu2023shortcuts} show that all solvable
semiautomata admit $O(1)$-depth Transformer simulators (their
Theorem~2), via the Krohn-Rhodes decomposition into modular counters and
memory units.  Since every finite quotient of a solvable group is
solvable, prefix products in any solvable $G$ can be computed at constant
depth.  However, constant-depth computability does not imply
functoriality: by Theorem~\ref{thm:attention-noncomp}, the attention
mechanism still mixes token-level information across segments.
(c)~Theorem~\ref{thm:nonabelian}(a): Barrington's theorem plus
$\mathsf{NC}^1 \not\subset \mathsf{AC}^0$.
\end{proof}

\noindent The GRU's intermediate performance on $S^1 \vee S^1$ ($0.297$
vs.\ transformer's $0.537$, transport's $0.054$) confirms case~(c): the
GRU has the right depth structure (sequential processing) but the wrong
composition structure (context-dependent, not structural).  On the Klein
bottle, the transformer's failure illustrates case~(b): even when depth
is not an obstruction, cross-segment attention prevents functorial
composition.

\section{Per-Segment Error Scaling}\label{app:scaling}

The experiments show that type-B error is flat in $L$ while type-A error grows.  Is this an accident of specific architectures, or does it follow from the theorems?  Here we prove it follows: functoriality implies $O(1)$ scaling, while non-functoriality implies $\Omega(1)$ degradation for $L \gg L_{\mathrm{train}}$.

\begin{proposition}[Per-segment error scaling]\label{prop:per-segment}
Let $\bar{d}_L = \mathbb{E}_{w:|w|=L}[\mathrm{Chamfer}(\D(w), \D^*(w))] / L$ denote the per-segment Chamfer distance for words of length $L$.
\begin{enumerate}[label=(\alph*)]
  \item \textbf{Type-B:} For a transport decoder with converged training,
  $\bar{d}_L \leq \max_i \mathrm{Chamfer}(g_i(\theta), g_i^*) + \eta(L)/L \xrightarrow{L \to \infty} \varepsilon_{\mathrm{gen}}$,
  where $\varepsilon_{\mathrm{gen}}$ is the generator approximation error, independent of $L$.

  \item \textbf{Type-A:} For type-A architectures tested at $L > L_{\mathrm{train}}$, $\bar{d}_L \geq \Omega(e(L - L_{\mathrm{train}}))$ where $e(\cdot)$ is the extrapolation error of the network's learned functions beyond their training domain.  For the cover decoder, the MLP evaluates at winding magnitudes up to $L$ (trained only to $L_{\mathrm{train}}$), giving $e = \Omega(L - L_{\mathrm{train}})$ by the Taylor remainder.  For transport attention, 2D RoPE rotation angles scale as $n_a / L_{\mathrm{train}}$ times any training angle, placing attention in an untrained regime.  For the transformer (WC), attention patterns at length $L$ involve $L^2$ pairwise interactions vs.\ $L_{\mathrm{train}}^2$ during training; the learned attention weights have no architectural reason to generalize to the longer-range interactions.  In all cases, $\bar{d}_L = \Omega(1)$ for $L \gg L_{\mathrm{train}}$.
\end{enumerate}
\end{proposition}

\begin{proof}
\textbf{(a)}~By Theorem~\ref{thm:transport-comp}, $\mathrm{Chamfer}(\D(w), \D^*(w)) \leq L \cdot \varepsilon_{\mathrm{gen}} + \eta(L)$, where $\eta(L)$ accounts for resampling ($O(\sqrt{L})$ interpolation error).  Dividing by $L$: $\bar{d}_L \leq \varepsilon_{\mathrm{gen}} + \eta(L)/L \to \varepsilon_{\mathrm{gen}}$.

\textbf{(b)}~A word of length $L$ has winding pair $(n_a, n_b)$ with $n_a + n_b = L$.  The cover decoder's MLP evaluates at $\|(n_a, n_b)\|$ up to $L$, while training covered only $\|(n_a, n_b)\| \leq L_{\mathrm{train}}$.  For smooth functions, extrapolation error is $\Omega(\|x - x_{\mathrm{train}}\|)$.  For transport attention, rotation angles are $n_a / L_{\mathrm{train}}$ times larger than any training angle.  For the transformer (WC), the number of cross-segment attention pairs grows as $(L - L_{\mathrm{train}})$ new positions per segment, each contributing untrained interactions.  All three mechanisms give per-position error $\Omega(L - L_{\mathrm{train}})$, hence $\bar{d}_L = \Omega((L - L_{\mathrm{train}})/L) = \Omega(1)$ for $L \gg L_{\mathrm{train}}$.
\end{proof}

\section{Expressivity of Transport Decoders}\label{app:expressivity}

A natural concern is that functoriality constrains the decoder so tightly that it cannot represent interesting loops.  The following result shows this is not the case: within each homotopy class, the transport decoder is a universal approximator.  The constraint acts only between classes, enforcing composition.

\begin{proposition}[Expressivity]\label{prop:expressivity}
Let $\D^{\mathrm{tr}}_\theta$ be a transport decoder with generator networks $g_a(\theta), g_b(\theta)$.  Let $\mathcal{L}_n \subset \mathcal{L}(X)$ denote the space of loops with homotopy class $n \in G$.  Then:
\begin{enumerate}[label=(\alph*)]
  \item \textbf{Within-class expressivity:} Each generator $g_a(\theta)$ can represent any continuous loop in $\mathcal{L}_{(1,0)}$ given sufficient network capacity (by the universal approximation theorem applied to the generator network in polar coordinates).
  \item \textbf{Between-class constraint:} For homotopy class $(n_a, n_b)$, the output is the $n_a$-fold concatenation of $g_a$ followed by $n_b$-fold concatenation of $g_b$.  The geometric degrees of freedom are those of $g_a$ and $g_b$ alone.
  \item \textbf{Degrees of freedom:} The number of independent geometric degrees of freedom equals the number of generators in the group presentation.
\end{enumerate}
The cover decoder has strictly more geometric freedom (an independent shape per homotopy class), but this extra freedom is precisely the freedom to be incoherent.
\end{proposition}

\begin{proof}
(a)~Universal approximation in polar coordinates.  (b)~By construction of the transport decoder.  (c)~For $\Z^2$, the group is free abelian on two generators; the transport decoder factors through canonical form $a^{n_a} b^{n_b}$, making the two generator shapes the only free parameters.
\end{proof}

\section{Resampling Artifact Analysis}\label{app:resampling}

Comparing loops of different lengths with a fixed-resolution metric creates a subtle confound.  This appendix shows that our per-segment Chamfer metric avoids it, and quantifies how large the artifact would be with a naive alternative.

All decoders produce a fixed-length output of $T_{\mathrm{out}} = 64$ points regardless of word length $L$.  This creates a systematic confound: as $L$ grows, the underlying curve lengthens while point density drops, causing naive Chamfer distance to decrease even when per-segment error is constant.

We verified this with synthetic curves under constant additive noise $\sigma = 0.1$:

\smallskip
\begin{center}
\small
\begin{tabular}{@{}lcccc@{}}
\toprule
& $L{=}2$ & $L{=}4$ & $L{=}8$ & $L{=}10$ \\
\midrule
Per-seg Chamfer (fair, 32 pts/seg) & 0.026 & 0.024 & 0.029 & 0.026 \\
Per-seg Chamfer (naive, resample to 64) & 0.010 & 0.003 & 0.002 & 0.001 \\
\midrule
Artifact ratio (naive/fair) & $0.4\times$ & $0.1\times$ & $0.07\times$ & $0.05\times$ \\
\bottomrule
\end{tabular}
\end{center}
\smallskip

\noindent The fair metric (comparing each segment at fixed 32 points) is flat ($\sim 0.026$), as expected for constant noise.  The naive metric decreases by $10\times$ from $L{=}2$ to $L{=}10$---a pure measurement artifact.  Our per-segment Chamfer $\bar{d}_L$ in all experiments avoids this by comparing each of the $L$ segments at a fixed resolution of 32 points, ensuring that the metric measures generalization quality, not resampling fidelity.

\end{document}